\documentclass[letterpaper]{article} 
\usepackage{aaai2027}  
\usepackage[hyphens]{url}  
\usepackage{graphicx} 
\usepackage{natbib}  
\usepackage{caption} 
\usepackage{algorithm}
\usepackage{algorithmic}

\usepackage{newfloat}
\usepackage{listings}
\DeclareCaptionStyle{ruled}{labelfont=normalfont,labelsep=colon,strut=off} 
\floatstyle{ruled}
\newfloat{listing}{tb}{lst}{}
\floatname{listing}{Listing}

\usepackage{booktabs}

\title{Trajectory-Aware Best-Arm Identification for Local Search Allocation\\ in Bayesian Optimization and Beyond}
\author{
    Nobuo Namura,
    Sho Takemori
}
\affiliations{
    Fujitsu Limited, Kawasaki, Japan\\
    \{namura.nobuo, takemori.sho\}@fujitsu.com
}

\usepackage{enumitem}
\usepackage{xcolor}
\usepackage{caption}

\usepackage{amssymb,amsmath,bm,amsthm}
\usepackage{cleveref}
\usepackage{dsfont}

\theoremstyle{plain}
\newtheorem{theorem}{Theorem}[section]
\newtheorem{proposition}[theorem]{Proposition}
\newtheorem{lemma}[theorem]{Lemma}
\newtheorem{corollary}[theorem]{Corollary}
\theoremstyle{definition}

\newtheorem{assumption}[theorem]{Assumption}
\theoremstyle{remark}
\newtheorem{remark}[theorem]{Remark}

\crefname{theorem}{Theorem}{Theorems}
\crefname{proposition}{Proposition}{Propositions}
\crefname{lemma}{Lemma}{Lemmas}
\crefname{definition}{Definition}{Definitions}
\crefname{remark}{Remark}{Remarks}

\DeclareMathOperator{\argmin}{argmin}

\newcommand{\RR}{\mathbb{R}}

\newcommand{\ex}[2][{}]{\mathbb{E}_{#1}\left[#2\right]}

\newcommand{\cA}{\mathcal{A}}

\newcommand{\cD}{\mathcal{D}}

\newcommand{\cH}{\mathcal{H}}

\newcommand{\cN}{\mathcal{N}}

\newcommand{\cX}{\mathcal{X}}

\begin{document}

\nocopyright

\maketitle

\begin{abstract}
Bayesian optimization (BO) is effective for expensive black-box optimization, but its performance often degrades on multimodal or high-dimensional problems.
A promising remedy is to run multiple local search processes, such as trust region--based BO, and allocate evaluations to those with high potential.
However, current best values do not always reflect the final performance of each process.
We propose a trajectory-aware best-arm identification (BAI) framework that uses optimization history to guide this allocation.
The proposed method extrapolates improvement trajectories, estimates final performance, and progressively eliminates low-potential candidates.
For trust region--based BO, we theoretically show that the proposed BAI-guided allocation accelerates convergence to the global optimum under mild assumptions.
Experiments on synthetic and real-world benchmarks demonstrate that our method improves BO on multimodal problems.
We also show that the same mechanism can enhance other local or population-based optimizers, suggesting its potential as a general extension strategy for multimodal black-box optimization.
\end{abstract}


\section{Introduction}
Bayesian optimization (BO) \cite{Mockus1978} is widely used for expensive black-box optimization, where each function evaluation is costly. Its applications include hyperparameter tuning, neural architecture search, molecular and materials design, and industrial product design. Standard BO uses Gaussian process (GP) regression, or Kriging \cite{matheron1963principles}, to model the objective function from limited observations. The resulting posterior mean and uncertainty define an acquisition function that balances exploitation and exploration for efficient global optimization \cite{jones1998efficient}.

Despite its success, conventional BO often struggles with multimodal functions containing many local optima, and high-dimensional problems. A key limitation is the length-scale parameter of GP kernels. Even with automatic relevance determination, the length scale is typically fixed within each dimension and can be dominated by local variations, making it difficult to capture global structure \cite{xiong2007non}. For example, on multimodal functions such as the Rastrigin function, a GP may fit short-period oscillations while missing the global quadratic trend, causing the acquisition function to be misled by local optima. Similarly, in high-dimensional problems, length scales are determined by nearby sample points, making global structure difficult to learn. This issue becomes more pronounced as optimization progresses and sample points become closer.

Several methods have been proposed to address these limitations. Hybrid surrogate models, such as universal Kriging, combine polynomial regression with GP regression to capture both global trends and local deviations, and related approaches have been applied to high-dimensional BO \cite{tan2023two}. Other methods reduce the effective dimensionality using random projections \cite{wang2016bayesian}, length-scale adaptation \cite{eriksson2021high}, or GP priors that favor larger length scales \cite{hvarfner2024vanilla}. These approaches mainly aim to improve global modeling by suppressing local structure or reducing the burden on GP regression.

Another direction is to improve local modeling by restricting the regression region. Trust region BO, such as TuRBO \cite{eriksson2019scalable}, builds local GP models within one or more trust regions (TRs) and adapts their sizes during optimization. TuRBO-$m$ uses multiple TRs to explore several regions in parallel, whereas TuRBO-1 uses a single TR and restarts when the TR becomes too small. This makes TuRBO suitable for high-dimensional and multimodal problems. In contrast to other high-dimensional BO methods that ignore local structure, TuRBO incrementally approximates the local landscape and is well-suited for finding local optima. This makes it potentially effective for both high-dimensional and multimodal optimization problems. However, challenges remain in achieving global optimization; in many experiments, TuRBO-$m$ and TuRBO-1 perform comparably, highlighting a gap in global search capability.

In this work, we propose a new algorithm that identifies global optima by leveraging multiple local search processes, such as TRs in TuRBO or independently initialized optimizers running in parallel. Since objective values typically improve over time and exhibit nonstationary behavior, we formulate the selection of promising search processes as a nonstationary best-arm identification (BAI) problem. Our method predicts the final objective value of each process from its optimization history and progressively eliminates candidates with low potential. Because it only requires optimization histories, the proposed mechanism is not tied to a specific surrogate model, acquisition function, or optimizer design.

The primary goal of this paper is to improve BO on multimodal and high-dimensional black-box optimization problems. Accordingly, our theoretical analysis focuses on trust region--based BO and shows that BAI-guided allocation over multiple TRs accelerates convergence to the global optimum under mild assumptions. We also evaluate the applicability of the same trajectory-aware BAI mechanism beyond BO, showing that it can improve other local or population-based optimizers.

Our main contributions are summarized as follows:

\begin{itemize}
\item We introduce a trajectory-aware BAI algorithm that leverages optimization history to allocate evaluations among multiple local search processes, including TuRBO and parallel optimizer instances, for efficient black-box optimization (Section \ref{sec:proposal}).
\item We theoretically show that BO with BAI-guided allocation over multiple TRs accelerates convergence to the global optimum, compared with conventional BO based on a single global GP regression (Section \ref{sec:theory}).
\item We empirically demonstrate that applying our method to TuRBO improves performance over the original TuRBO and other state-of-the-art algorithms on both synthetic multimodal functions and real-world problems (Section \ref{sec:experiment_bo}).
\item We further show that the same trajectory-aware BAI mechanism can improve other optimization algorithms, suggesting its potential as a general extension strategy for multimodal black-box optimization (Section \ref{sec:experiment_ea}).
\end{itemize}


\section{Related Works}
\subsection{High-Dimensional Bayesian Optimization}
High-dimensional BO has been studied extensively to prevent GP regression from becoming overly complex. Existing methods can be broadly categorized into dimensionality reduction, length-scale adjustment, additive structure modeling, and TR-based methods.

Dimensionality reduction methods embed the original search space into a lower-dimensional subspace. Representative examples include random embedding BO (REMBO) \cite{wang2016bayesian} and hashing-enhanced subspace BO (HeSBO) \cite{nayebi2019framework}, both of which use random linear projections. BO via incumbent-guided direction lines and subspace embeddings (BOIDS) \cite{ngo2025boids} further combines subspace embeddings with line search strategies, as in LineBO \cite{kirschner2019linbo}.

Length-scale adjustment has also attracted recent attention. Sparse axis-aligned subspace BO (SAASBO) \cite{eriksson2021high} induces sparsity in inverse length scales, effectively reducing the active dimensionality. The dimensionality-scaled length-scale prior (DSP) \cite{hvarfner2024vanilla} instead encourages larger length scales in high dimensions so that the GP captures global structure. Similar effects have been reported for classical maximum likelihood estimation of length scales in Kriging \cite{papenmeier2025understanding}. Coordinate backoff BO (CobBO) \cite{tan2023two} separates global and local modeling by combining a coarse GP for global structure with another GP for fine-grained low-dimensional subspaces.

Additive structure methods assume that the objective decomposes into a sum of functions over separable variable groups. Additive GP \cite{duvenaud2011additive} is a representative approach in this category.

TR-based methods improve local modeling by restricting GP regression to local regions. TuRBO is a representative method that performs localized GP regression within progressively adapted TRs. Recent extensions include Bounce \cite{papenmeier2023bounce}, which combines TRs with dimensionality reduction and handles discrete variables, and efficient restart strategies for improving global search \cite{namura2025rei}.

\subsection{Non-Stationary Best Arm Identification}
The multi-armed bandit (MAB) problem is an online decision-making framework in which a learner repeatedly selects an arm and observes a reward or loss \cite{lattimore2020bandit}. Best-arm identification (BAI) in the fixed-budget setting aims to identify the best arm with a limited number of pulls while minimizing the probability of misidentification \cite{audibert2010best}.

In this paper, we formulate the selection of TRs or parallel local search processes as a non-stationary BAI problem. The observed losses are generated by optimizers such as TuRBO and typically decrease as evaluations are allocated. This setting corresponds to rising bandits, where the expected reward increases, or equivalently the expected loss decreases, with the number of pulls \cite{tekin2012online,heidari2016tight,li2020efficient,cella2021best,mussi2024best,takemori24bai}. Under a loss model, \citet{takemori24bai} proposed sequential-halving-based methods \cite{karnin2013almost} and misidentification-probability bounds. Compared with more general model-free settings \cite{mussi2024best}, such model-based methods can achieve smaller misidentification probabilities when the model assumption is appropriate.

TuRBO-$m$ selects a TR via Thompson sampling (TS) \cite{thompson1933likelihood}. However, because the observed losses are non-stationary, TS does not provide theoretical guarantees in this setting. In contrast, rising-bandit methods \cite{takemori24bai,mussi2024best} are designed to control the probability of misidentifying the best arm.

\section{Preliminaries}
We consider the folowing problem of minimizing an expensive black-box objective function $f$ defined over continuous variables $\mathbf{x}$:
\begin{equation}
    \label{eq:min-prob}
    \mathbf{x}_{*} = \argmin_{\mathbf{x} \in \Omega} f(\mathbf{x}),
\end{equation}
where $\Omega = [0,1]^{D}$ denotes the $D$-dimensional normalized input space.
In the following, we briefly review the foundations of BO and the TuRBO framework.

\subsection{Bayesian Optimization}
In BO, a GP regression is trained using a set of $t$ samples $\mathcal{D}_t = \{ (\mathbf{x}^{(i)}, y^{(i)}) \}_{i=1}^{t}$, where each sample consists of a design variable $\mathbf{x}^{(i)}$ and its corresponding objective function value $y^{(i)}=f(\mathbf{x}^{(i)})$. The GP models the objective function as using a mean function $m_{\theta}(\mathbf{x})$ and a kernel function $k_{\theta}(\mathbf{x}, \mathbf{x}')$ that defines the covariance between function values. Given the training data $\mathcal{D}_t$, the GP posterior $\mathcal{GP}(m_{\theta}(\mathbf{x}), k_{\theta}(\mathbf{x}, \mathbf{x}')|\mathcal{D}_t)$ provides a predictive distribution at any input $\mathbf{x}$, yielding the estimated mean $\hat{f}(\mathbf{x} \mid \mathcal{D}_t)$ and variance $\hat{\sigma}^2(\mathbf{x} \mid \mathcal{D}_t)$. In most cases, the GP hyperparameters $\theta$ are determined by either maximum a posteriori (MAP) estimation or MLE. In this work, we standardize the objective values in $\mathcal{D}_t$ before training the GP so that the mean function is fixed to zero, i.e., $m_{\theta}(\mathbf{x}) = 0$.

Using the GP regression, BO identifies input points $\mathbf{x}$ where the predicted mean $\hat{f}(\mathbf{x} \mid \mathcal{D}_t)$ is low—indicating a potentially small objective value—and the predicted variance $\hat{\sigma}^2(\mathbf{x} \mid \mathcal{D}_t)$ is high—indicating high uncertainty in the GP regression. The former corresponds to exploitation, while the latter corresponds to exploration. To balance these two objectives, BO selects the next query point $\mathbf{x}^{(t+1)}$ by optimizing a criterion known as an acquisition function. Representative acquisition functions include expected improvement (EI) \cite{Mockus1978}, which maximizes the expected value of objective function's improvement over the current best value $\min_{i} y^{(i)}$, and TS, which minimizes a function sampled from the GP posterior distribution $\hat{f} \sim \mathcal{GP}(0, k_{\theta}(\mathbf{x}, \mathbf{x}')|\mathcal{D}_t)$.

\subsection{Trust Region Bayesian Optimization}
In TuRBO-$m$, for the $j$-th TR $(j=\{1,\dots,m\})$, $n$ initial samples $\mathcal{D}^{(j)}$ are generated uniformly at random. Based on these, the center point
$\mathbf{c}^{(j)} = \argmin_{\mathbf{x} \in \mathcal{D}^{(j)}} f(\mathbf{x})$ 
is identified, and the initial trust region $\mathcal{TR}^{(j)}$ is defined as a hypercube centered at $\mathbf{c}^{(j)}$ with side length $l^{(j)} = l_0$ along each dimension. In practice, the side lengths are adjusted based on the GP regression's length scales. 
Within each $\mathcal{TR}^{(j)}$, GP regression is trained using the associated sample set $\mathcal{D}^{(j)}$. Then, using TS, a batch of $b$ functions $\hat{f}^{(j,k)}$ (for $k \in \{1, \dots, b\}$) is sampled from the GP posterior. For each function, the solution minimizing the sampled function within the TR is obtained as 
$\mathbf{x}_*^{(j,k)} = \argmin_{\mathbf{x} \in \mathcal{TR}^{(j)}} \hat{f}^{(j,k)}(\mathbf{x})$ 
along with its corresponding predicted value $\hat{f}_{*}^{(j,k)}$.
Once all $m \times b$ candidates are obtained across all TRs, the $b$ solutions $\mathbf{x}_*^{(j,k)}$ with the lowest $b$ estimated objective values $\hat{f}_{*}^{(j,k)}$ are selected. These selected points are then evaluated using the actual (expensive) objective function and added to the sample set $\mathcal{D}^{(j)}$ of their respective TRs.

Each time new sample points are added to a TR, the algorithm checks whether the best objective value within that TR has been improved. Based on the outcome, the TR size $l^{(j)}$ is updated according to the following rules. If the objective function has not improved for $\tau_{\text{fail}}$ consecutive iterations, the TR size $l^{(j)}$ is halved. Conversely, if the objective has improved for $\tau_{\text{succ}}$ consecutive iterations, $l^{(j)}$ is doubled. Additionally, when an improvement is observed, $\mathbf{c}^{(j)}$ is shifted to the location of the new best solution. 
This adaptive strategy continues until the size of the TR becomes smaller than a predefined threshold $l_{\min}$. Once this condition is met, local search within the TR is terminated and a new TR is initialized with a fresh set of initial samples.

\section{Proposed Method}
\label{sec:proposal}
We propose a trajectory-aware BAI algorithm that extends the model-based BAI framework of \citet{takemori24bai} to multimodal black-box optimization. 
The proposed algorithm treats multiple local search processes as arms and predicts their final objective values from optimization histories. 
In this work, we mainly instantiate this framework with TuRBO-$m$, where each TR corresponds to an arm, for consistency with our theoretical analysis. 
Because the mechanism only requires optimization histories, it can also be applied to other multi-start or population-based optimization methods, such as evolutionary algorithms, gradient-based methods, and the Nelder--Mead method.

\subsection{Optimization Trajectory Prediction}
\label{subsec:trajectory}
For the $m$ independently operating optimizers, let the optimization history of the $j$-th optimizer be denoted by $\mathcal{D}^{(j)} = \{(\mathbf{x}^{(i,j)}, y^{(i,j)})\}_{i=1}^{t}$. In the prior work \cite{takemori24bai}, ridge and lasso regressions were applied to an idealized version of this history to predict the final objective value based on the average trend of the observations. However, most of the samples in the actual optimization history have very large objective values and lie far from the optimum, while only a small subset of samples reflects the landscape near the current best solution with favorable objective values. As a result, regression using all the sample points in $\mathcal{D}^{(j)}$ is heavily influenced by the global structure of the objective function, and tends to produce nearly identical predictions for all $m$ optimizers. To appropriately reflect the landscape specific to the subspace explored by each optimizer and accurately model the optimization trajectory, it is necessary to use the history of minimum values defined as follows:

\begin{equation*}
    \label{eq:history}
    \mathcal{D}_{*}^{(j)} = \left\{ (\mathbf{x}^{(i,j)}, y^{(i,j)}) \in \mathcal{D}^{(j)} \;\middle|\; y^{(i,j)} < \min_{k < i} y^{(k,j)} \right\}.
\end{equation*}

\noindent We approximate this sequence using the following function:

\begin{equation}
    \label{eq:loss}
    \tilde{f}_*^{(\tau,j)} = \sum_{s=1}^{d}\frac{\beta_{s}^{(j)}}{\tau^{\rho_s}},
\end{equation}

\noindent where, $\tau$ denotes the time index corresponding to $t$, $\beta_{s}^{(j)}$ $(s \in \{1,\dots,d\})$ are trainable parameters. Following \cite{takemori24bai}, the basis functions include a constant term ($\rho_s=0$), and the remaining $\rho_s$ are fixed decay constants.
Assuming ridge regression with a regularization parameter $\lambda$, we define $\mathbf{z}(\tau) = (\tau^{-\rho_1}, \dots, \tau^{-\rho_d})^\top$ and use it to predict the trajectory of best objective values at arbitrary time steps $\tau$ as follows:

\begin{align} 
    \label{eq:ridge}
        \tilde{f}_*^{(\tau,j)} &= \mathbf{z}^{\top}(\tau)(V^{(j)})^{-1} \sum_{i \in \mathcal{D}_{*}^{(j)}} y_*^{(i,j)}\mathbf{z}(i), \\
        V^{(j)} &= \lambda I_{d} + \sum_{i \in \mathcal{D}_{*}^{(j)}} \mathbf{z}(i)\mathbf{z}^{\top}(i),
\end{align}

\noindent where $y_*^{(i,j)}$ is $y^{(i,j)}$ in $\mathcal{D}_{*}^{(j)}$ and $I_{d}$ is the $d$-dimensional identity matrix.

While the regression described above is applicable to ideal optimization histories, the actual best-so-far history $\mathcal{D}_{*}^{(j)}$ obtained from a single trial often consists of extremely sparse and non-smooth points, making regression challenging. Furthermore, in batch or population-based optimizers such as TuRBO and evolutionary algorithms, where multiple sample points are added simultaneously, the order of these points carries no meaningful information. Therefore, appropriate preprocessing is needed. We apply the following steps to each history to address these issues. Steps 1 and 2 are applied before generating $\mathcal{D}_{*}^{(j)}$, and step 3 is applied afterward:
\begin{enumerate}
\item Sample points added simultaneously are sorted in ascending order of their objective function values.
\item From each $\mathcal{D}^{(j)}$, only the sample point with the minimum objective value among the $n$ initial samples is retained, and the rest are removed.
\item The median of the best-so-far history of $\mathcal{D}_{t}$ is computed, and any sample point in $\mathcal{D}_{*}^{(j)}$ whose objective value is greater than or equal to this median is removed, as long as at least one point remains afterward.
\end{enumerate}
\noindent The effects of these steps and the validity of the resulting trajectory prediction are detailed in Appendix \ref{apx:prediction}.

\subsection{TuRBO-$m$ with Trajectory-Aware BAI}
We now describe the TuRBO-based instantiation of the proposed framework. 
To select the search process that is likely to yield the best solution, we employ the sequential halving (SH) algorithm \cite{karnin2013almost}. 
Algorithm \ref{alg:nsbai} outlines TuRBO-$m$-BAI, where TuRBO-$m$ is used as the base optimizer and each TR is treated as an arm.
Similar to standard TuRBO-$m$, the procedure begins by generating $n$ initial samples for each of the $m$ TRs. Subsequently, the SH algorithm is used to identify the most promising TR (i.e., the best arm) by iteratively sampling from the current set of candidates.

In SH, the set of surviving arms at round $r$, denoted as $A_r \subseteq \{1,\dots,m\}$, is reduced to a smaller set $A_{r+1}$ by halving its size $|A_{r+1}| = \lceil |A_r| / 2 \rceil$, based on performance. As a result, the best arm can be identified within at most $n_{\text{SH}}$ $(\geq b \cdot m\lceil\log_2 m\rceil)$ additional samples. In our method, this elimination step uses the predicted objective values $\tilde{f}_*^{(\tilde{T},j)}$, derived from Eq. \eqref{eq:ridge}, at the following fixed horizon $\tau = \tilde{T}$. 

\begin{algorithm}[H]
\caption{TuRBO-$m$-BAI}
\label{alg:nsbai}
\textbf{Input}: Objective function $f$, Variable space $\Omega$, Initial budget $n$, Arm-identification budget $n_{\text{SH}}$, Total budget $N$ \\
\textbf{Parameter}: Initial TR length $l_0$, Minimum TR length $l_{\min}$, Number of TRs $m$, Batch size $b$, Success threshold $\tau_{\text{succ}}$, Failure threshold $\tau_{\text{fail}}$ \\
\textbf{Output}: Samples $\mathcal{D}_N = \{ \mathbf{x}^{(i)}, y^{(i)}\}_{i=1}^{N}$ \\
\begin{algorithmic}[1] 
\STATE $\mathcal{D}_0 = \varnothing$, $t = 0$
\FOR{$j=1$ to $m$}
    \STATE Generate initial samples $\mathcal{D}^{(j)} = \{ \mathbf{x}^{(i)}, y^{(i)}\}_{i=1}^{n}$
    \STATE Initialize TR $\mathcal{TR}^{(j)}$ with length $l^{(j)}=l_0$ centered around $\mathbf{c}^{(j)}=\argmin_{\mathbf{x} \in \mathcal{D}^{(j)}} f(\mathbf{x})$
    \STATE Augment $\mathcal{D}_{t+n} \leftarrow \mathcal{D}_{t} \cup \mathcal{D}^{(j)}$
    \STATE $t \leftarrow t + n$
\ENDFOR

\STATE Initialize arm set $A_1 = \{1, \dots, m\}$
\FOR{$r=1$ to $\lceil \log_2 m \rceil$}
    \FOR{$j \in A_r$}
        \FOR{$i=1$ to $\left\lfloor n_{\text{SH}} / (b|A_r| \lceil \log_2 m \rceil ) \right\rfloor$}
            \STATE Train a GP regression and sample $b$ functions $\{\hat{f}^{(j,k)}\}_{k=1}^b$ from the GP posterior using TS
            \STATE Minimize each $\hat{f}^{(j,k)}$ within $\mathcal{TR}^{(j)}$ and obtain next samples $\{(\mathbf{x}^{(t+i)}, y^{(t+i)})\}_{i=1}^b$ 
            \STATE Augment $\mathcal{D}^{(j)} \leftarrow \mathcal{D}^{(j)} \cup \{(\mathbf{x}^{(t+i)}, y^{(t+i)})\}_{i=1}^b$
            \STATE Augment $\mathcal{D}_{t+b} \leftarrow \mathcal{D}_{t} \cup \mathcal{D}^{(j)}$
            \STATE $t \leftarrow t + b$
            
            
            \STATE Update $\mathbf{c}^{(j)}$ and $l^{(j)}$ of $\mathcal{TR}^{(j)}$
        \ENDFOR
        \STATE Estimate $\tilde{f}_*^{(\tilde{T},j)}$ by Eq. \eqref{eq:ridge} using $\mathcal{D}^{(j)}$
    \ENDFOR
    \STATE Select $A_{r+1}$ that includes $\lceil |A_r| / 2 \rceil$ TRs with the lowest predictions $\tilde{f}_*^{(\tilde{T},j)}$ from $A_r$
\ENDFOR
\STATE Let $j$ be the final selected TR in $A_{\lceil \log_2 m \rceil + 1}$
\WHILE{$t < N$}
    \STATE Train a GP regression and sample $b$ functions $\{\hat{f}^{(j,k)}\}_{k=1}^b$ from the GP posterior using TS
    \STATE Minimize each $\hat{f}^{(j,k)}$ within $\mathcal{TR}^{(j)}$ and obtain next samples $\{(\mathbf{x}^{(t+i)}, y^{(t+i)})\}_{i=1}^b$ 
    \STATE Augment $\mathcal{D}^{(j)} \leftarrow \mathcal{D}^{(j)} \cup \{(\mathbf{x}^{(t+i)}, y^{(t+i)})\}_{i=1}^b$
    \STATE Augment $\mathcal{D}_{t+b} \leftarrow \mathcal{D}_{t} \cup \mathcal{D}^{(j)}$
    \STATE $t \leftarrow t + b$
    \STATE Update $\mathbf{c}^{(j)}$ and $l^{(j)}$ of $\mathcal{TR}^{(j)}$
\ENDWHILE
\STATE \textbf{return} $\mathcal{D}_{t}$
\end{algorithmic}
\end{algorithm}

\begin{equation*}
\label{eq:target}
\begin{aligned}
\tilde{T}
&= N - (m - 1)n \\
&\quad - b\left(
    \sum_{r=1}^{\lceil \log_2 m \rceil}
    (|A_r| - 1)
    \left\lfloor
        \frac{n_{\text{SH}}}
        {|A_r| \lceil \log_2 m \rceil}
    \right\rfloor
\right),
\end{aligned}
\end{equation*}

\noindent where $N$ is the total evaluation budget. $\tilde{T}$ corresponds to the total number of initial and additional samples to be used for the final selected TR. Namely, the $\lceil |A_r| / 2 \rceil$ TRs with the lowest $\tilde{f}_*^{(\tilde{T},j)}$ are retained in $A_{r+1}$ for the next round.

After the best arm, i.e., the most promising TR in this TuRBO instantiation, is identified, the remaining budget is allocated to that TR following standard TuRBO-1 procedures. During the SH-based arm selection process, each TR is allowed to adapt its length (i.e., expand or shrink) in accordance with TuRBO's standard TR adjustment rules. However, to preserve the validity of arm selection, restarts of TRs are disallowed throughout the entire procedure. Specifically, the size of each TR is restricted such that $l^{(j)} \geq l_{\min}$, avoiding TR resets that would invalidate the selection of the optimal arm.

It should be noted that SH only determines which arm to select, and how much of the total budget $N$ should be allocated to the selection phase ($n_{\text{SH}}$) versus final exploitation is a non-trivial design decision. In this work, we empirically investigate this trade-off and provide practical guidance based on numerical experiments in Appendix \ref{apx:parametric}.

\section{Theoretical Analysis}
\label{sec:theory}
\newcommand{\lgs}{\alpha}
\newcommand{\simplereg}{R}
\newcommand{\numtr}{m}
\newcommand{\algours}{\cA_{\mathrm{ours}}}
\newcommand{\nsh}{n_{\mathrm{SH}}}
\newcommand{\bfx}{\mathbf{x}}
\newcommand{\bfz}{\mathbf{z}}
\newcommand{\trr}{\mathcal{TR}}
\newcommand{\reggp}{\tilde{\lambda}}

This section provides a theoretical proof demonstrating that the BAI-based TuRBO-$m$ algorithm achieves faster convergence compared to standard methods employing global GPs.
All the omitted proofs and technical details can be found in Appendix \ref{apx:theory}.
\subsection{(Theoretical) Problem Formulation}
Let $k: \Omega \times \Omega \rightarrow \RR$ be a positive definite kernel on $\Omega$ 
and $\cH_k(\Omega)$ be the associated reproducing kernel Hilbert space (RKHS).
For $g \in \cH_k(\Omega)$, the RKHS norm of $g$ is denoted by $\|g\|_k$.
We refer to \cite[Chapter 10]{wendland2004scattered} for basic properties of RKHS.
In this section, we focus on the Mat\'ern-$\nu$ kernel $k_{\lgs}$ with a smoothness parameter $\nu > 0$ and length scale $\lgs > 0$ defined as 
\begin{math}
   k_\alpha(x, x') :=  \frac{2^{1- \nu}}{\Gamma(\nu)} 
    r^{\nu} B_{\nu}(r)
\end{math}
where $r = \sqrt{2\nu}\|\bfx - \bfx'\|/\lgs $,
$B_\nu$ is the modified Bessel function of the second kind.
We omit $\nu$ from the notation $k_\alpha$ since we fix $\nu$ in our analysis.

\subsection{Objective and Evaluation Metric}
For $i=1,\dots, N$, a Bayesian optimization algorithm $\cA$ selects a point $\bfx^{(i)} \in \Omega$
and observes a noisy loss $y^{(i)} = f(\bfx^{(i)}) + \epsilon^{(i)}$, where $\{\epsilon^{(i)}\}_
{i=1}^{N}$ is an i.i.d. $R$-subgaussian noise sequence (e.g., $\epsilon^{(i)} \sim \cN(0, R^2)$) with a constant $R > 0$.
Then, at round $i$, $\cA$ computes an estimate $\hat{\bfx}_{*}^{(i)}$ of the optimal point based on the observations $\{(\bfx^{(s)}, y^{(s)})\}_{s=1}^{i}$.
The performance of an algorithm $\cA$ is measured by the expected simple regret 
\begin{equation}
    \label{eq:simple-reg-bound}
\simplereg_i(\cA; \Omega, f) = \ex{f(\hat{\bfx}_{*}^{(i)}) - f(\bfx_*)}.
\end{equation}

\subsection{Known Results on Simple Regret Bounds}
We introduce a known upper bound of the expected simple regret of a nearly optimal algorithm \cite{vakili2021optimal}.
For any positive definite kernel $k: \widetilde{\Omega} \times \widetilde{\Omega} \rightarrow \RR$ 
with a compact domain $\widetilde{\Omega} \subset \RR^D$ and $i \ge 1$,
we denote by $\gamma_i = \gamma_i(k, \widetilde{\Omega})$ the maximal information gain \cite{srinivas2010gaussian}
after observing $i$ points.
Then, 
a nearly optimal theoretical result for simple
regret minimization in the noisy setting is as follows:
there exists an algorithm $\cA$ satisfying the bound below \citep{vakili2021optimal,li2022gaussian} 
(see also Proposition \ref{prop:pe} in Appendix):
\begin{equation}
    \label{eq:standard-simple-regret}
    \simplereg_N(\cA; \Omega, f)
    = \widetilde{O}\!\left(
        \sqrt{\frac{\gamma_N(k, \Omega)}{N}}\,\bigl(\|f\|_k + R\bigr)
    \right),
\end{equation}
where the notation $\widetilde{O}$ absorbs $\mathrm{poly}(\log N, \log (1 + \|f\|_k))$ factors
and $R$ is the scale of noise $\epsilon^{(i)}$.
For later purposes, we let $U_N(k, \Omega, \|f\|_k) := \sqrt{\frac{\gamma_N(k, \Omega)}{N}} (\|f\|_k + R)$.

\subsection{BAI-based Algorithm}
In our theoretical analysis, we consider the following method, which is a simplified version of Algorithm \ref{alg:nsbai}.      
For $j=1, \cdots, m$, we let $\trr^{(j)} \subset \Omega$ be a trust region such that $\Omega = \cup_{1\le j \le \numtr} \trr^{(j)}$. We define $f^{(j)}$ by $f|_{\trr^{(j)}}$, i.e., $f^{(j)}$ is the restriction of $f$ to $\trr^{(j)}$.
Let $\cA^{(j)}$ be a BO algorithm that works on $\trr^{(j)}$.
As in Algorithm \ref{alg:nsbai}, for the first $\nsh$ rounds, we run SH using the loss model as in Eq. \eqref{eq:ridge} (with $y_*^{(i, j)}$ replaced by $y^{(i)}$) and select the trust region $\trr^{(j)}$ as in lines 9-22 in Algorithm \ref{alg:nsbai}.
Then, using the remaining budget, we run $\cA^{(j)}$ on the selected trust region $\trr^{(j)}$.

\begin{figure}[t]
    \centering  
    \includegraphics[width=0.85\columnwidth, height=3cm]{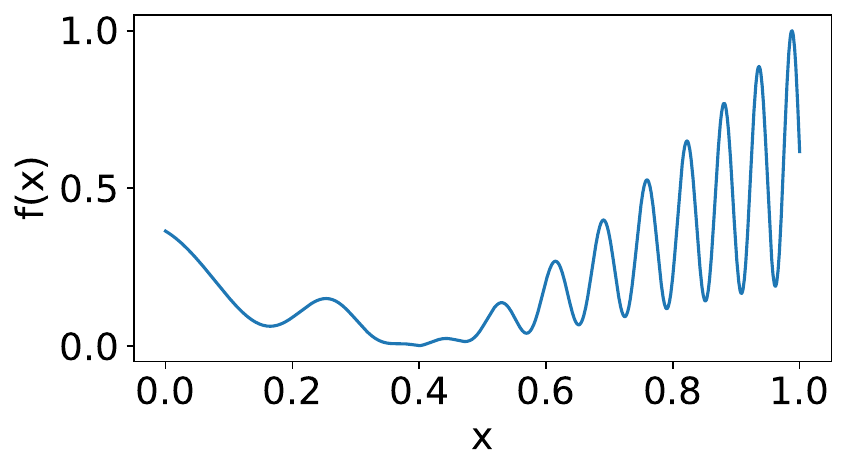}
    \caption{
    A synthetic objective on $\Omega=[0,1]$ with different length scales on $\trr^{(1)}=[0,0.5]$ and $\trr^{(2)}=[0.5,1]$.
    }
    \label{fig:example-diffrent-ls}
\end{figure}

\subsection{Assumptions}
For conciseness, we defer detailed assumptions to Appendix \ref{sec:assumptions} and provide informal assumptions in this section.
We assume the objective function has a different length scale on each $\trr^{(j)}$ ($1 \le j \le m$) as illustrated in Fig. \ref{fig:example-diffrent-ls}.
More formally, the length scale of $f$ is denoted by $\alpha > 0$ (i.e., $f \in \cH_{k_{\alpha}}(\Omega)$)
and the kernel and length scale of $f^{(j)}$ restricted to $\trr^{(j)}$ are denoted by $k^{(j)}$ and $\alpha^{(j)}> 0$ for $1 \le j \le \numtr$.
Moreover, for each $1 \le j \le \numtr$, we assume $\cA^{(j)}$ is a nearly optimal that satisfies the regret bound Eq. \eqref{eq:standard-simple-regret} and the assumptions of the loss model Eq. \eqref{eq:loss}.


\subsection{Main Theoretical Results}
As we see in Eq. \eqref{eq:standard-simple-regret}, the norm $\|f\|_k$ represents the complexity of the BO problem,
and the following result shows the complexity is reduced by the restriction.
This formalizes the intuition that in Fig. \ref{fig:example-diffrent-ls}
the BO problem on $\trr^{(1)}$ is easier than the original optimization problem on $\Omega$.
\begin{proposition}
    \label{prop:norm-comparison}
    For any $1\le j \le \numtr$,  
    we have $\|f^{(j)}\|_{k^{(j)}} \lesssim (\alpha/\alpha^{(j)})^{D/2} \|f\|_{k} $. 
    Here the notation $\lesssim$ hides constants independent of $f$, $\alpha$, and $\{\alpha^{(j)}\}_{1 \le j \le \numtr}$.
\end{proposition}
Then, by this proposition, we can prove the following.
\begin{theorem}
    \label{thm:simple-regret-bound}
    We define $j_* \in [m]$ as the index of the trust region $\trr^{(j_*)}$ containing $\bfx_*$.
    We simply denote $U_N(k_{\alpha}, \Omega, \|f\|_{k_{\alpha}})$ by $U_N$.
    With an appropriate choice of parameters and the BAI budget $\nsh = O(\log N)$,
    the expected simple regret of the proposed method is given as 
    \begin{equation}
        \label{eq:regret-bound-main}
        \simplereg_N(\algours; \Omega, f) = 
         \widetilde{O}\left(
            \left(\alpha / \alpha^{(j_*)}\right)^{D/2} U_{N-\nsh}
        \right).
    \end{equation} 
\end{theorem}
Since the BAI phase costs only $\nsh = O(\log N)$ rounds, Theorem \ref{thm:simple-regret-bound} shows that the expected simple regret can be significantly smaller than the standard bound $U_N$ whenever the length scale $\alpha^{(j_*)}$ around $\bfx_*$ exceeds the global one $\alpha$. Although the analyzed procedure only simplifies Algorithm \ref{alg:nsbai}, this already establishes the core benefit of \emph{adaptively} selecting the trust region: the method efficiently identifies $\trr^{(j_*)}$ containing $\bfx_*$ and inherits its reduced regret bound, which no non-adaptive choice can guarantee.


\section{Experiments on Bayesian Optimization}
\label{sec:experiment_bo}

We first evaluate the proposed trajectory-aware BAI framework in the context of trust region--based BO.
Specifically, we apply the proposed method to TuRBO-$m$ and compare TuRBO-$m$-BAI with state-of-the-art BO baselines.
The experiments are conducted on four 10-dimensional synthetic functions and two real-world optimization tasks.
In the main text, we report the results with $r_{\mathrm{SH}}=0.9$, where $r_{\mathrm{SH}}$ denotes the fraction of the total budget used for initialization and sequential halving.

Implementation details, detailed descriptions of the benchmark problems, detailed experimental results including the distributions of the final objective values and statistical tests, computational time, an ablation study and demonstration of trajectory prediction, parametric studies such as those on other values of $r_{\mathrm{SH}}$, and results on additional benchmark problems are provided in the Appendix.

\subsection{Experimental Setup}

The synthetic benchmarks are the Schwefel, Rastrigin, Ackley, and Rosenbrock functions. Schwefel and Rastrigin are highly multimodal, whereas Ackley and Rosenbrock are included to assess performance on weakly multimodal or nearly unimodal landscapes. 
The real-world benchmarks are the Robot Pushing and Rover Trajectory problems \cite{wang2018batched}, both converted into minimization tasks by negating the original objectives. In both cases, many trials converge before reaching the global optimum, suggesting complex multimodal landscapes with strong local optima (see Appendix \ref{apx:detailed_results}).

We compare TuRBO-$m$-BAI with TuRBO-1, TuRBO-$m$, CobBO, Bounce, DSP, GP-EI, and GP-TS. 
On the synthetic benchmarks, we set $N=1000$, $m=5$, and $b=1$; on the real-world benchmarks, we use $N=5000$, $m=10$, and $b=20$. 
The number of initial samples is set to $n=0.1N/m$ for TuRBO-based methods, while $m=1$ is assumed for the other methods. 
Since GP-EI and DSP both use LogEI \cite{ament2023unexpected} and CobBO does not support batch sampling, we fix $b=1$ for these methods. Due to computational constraints, we set $N=2000$ for GP-EI and $N=1000$ for DSP on the real-world benchmarks. 

Each method is evaluated over 31 independent trials, and the main performance metric is the average best-so-far objective value across those trials.

\begin{figure*}[t]
    \centering
    \includegraphics[width=0.95\textwidth]{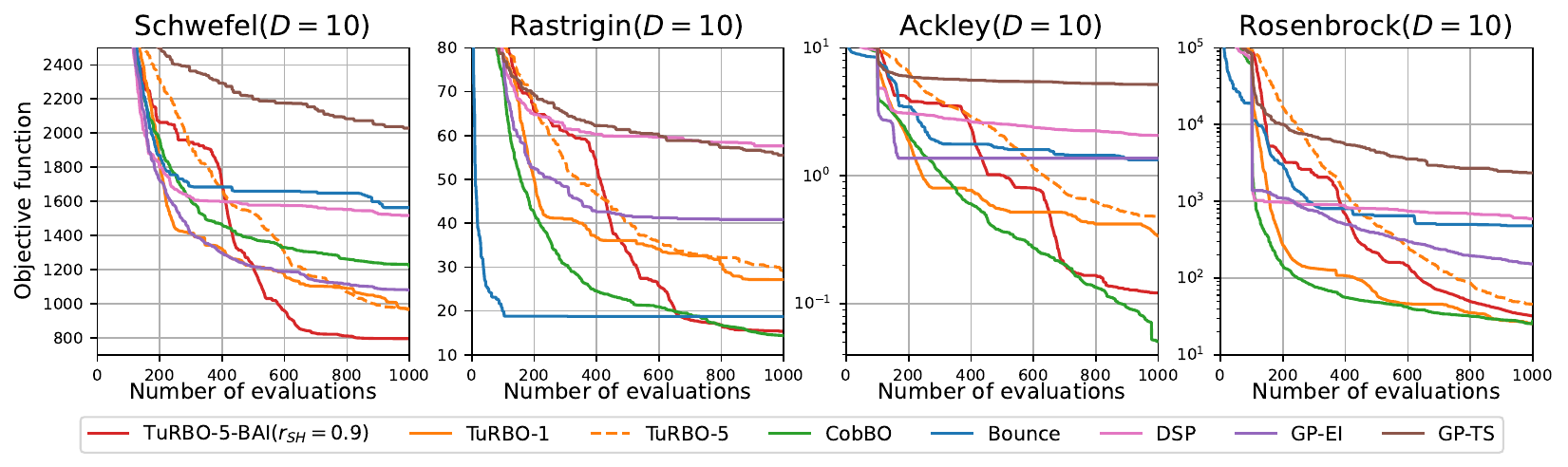}
    \caption{Comparison of TuRBO-5-BAI with $r_{\mathrm{SH}}=0.9$ against baselines on four synthetic test functions.}
    \label{fig:synthetic}
\end{figure*}

\begin{figure}[t]
    \centering
    \includegraphics[width=0.95\columnwidth]{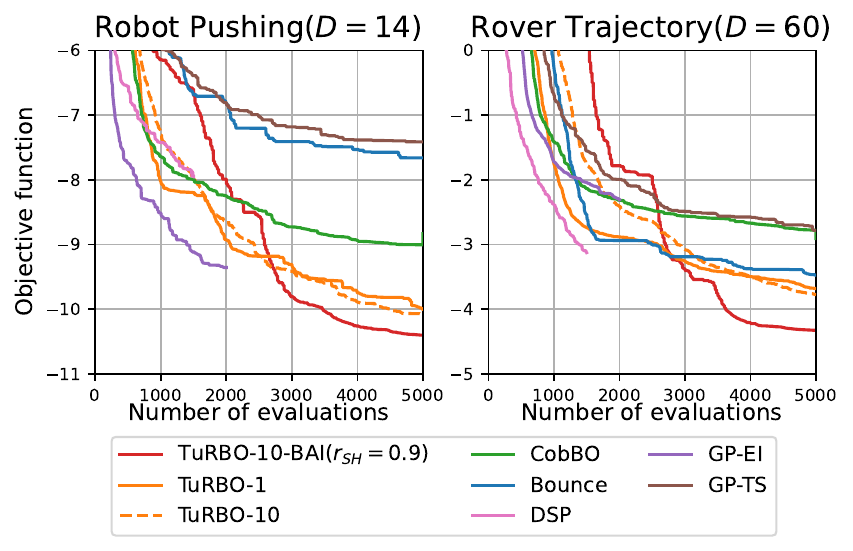}
    \caption{Comparison of TuRBO-10-BAI with $r_{\mathrm{SH}}=0.9$ against baselines on two real-world problems.}
    \label{fig:real}
\end{figure}

\subsection{Results on Synthetic Functions}

Figure \ref{fig:synthetic} shows the optimization histories on the synthetic benchmarks.
TuRBO-5-BAI substantially improves over both TuRBO-1 and TuRBO-5 on the highly multimodal Schwefel and Rastrigin functions.
For Schwefel, TuRBO-5-BAI achieves the best performance among all methods, showing that BAI-guided allocation is effective when local regions have markedly different potential. This result is consistent with our theoretical analysis, as the global and near-optimal regions in Schwefel have relatively large length scales.

TuRBO-5-BAI also clearly improves over the original TuRBO variants on Rastrigin and performs competitively with the strongest baselines.
Bounce can converge unusually fast on Rastrigin because its dimensionality reduction matches the optimum structure, making the result problem-specific rather than broadly comparable.

On Ackley and Rosenbrock, where the advantage of selecting among multiple local regions is less pronounced, TuRBO-5-BAI remains competitive.
Overall, these results suggest that the proposed BAI mechanism is particularly effective on strongly multimodal problems while remaining robust on less favorable problem classes.

\subsection{Results on Real-World Problems}

Figure \ref{fig:real} presents the results on the real-world benchmarks.
TuRBO-10-BAI achieves the best performance on both Robot Pushing and Rover Trajectory.
In Robot Pushing, several baselines stagnate around suboptimal objective values, whereas TuRBO-10-BAI continues improving and reaches substantially lower values.
A similar trend is observed in Rover Trajectory, where TuRBO-10-BAI escapes the local convergence behavior seen in many competing methods.

Overall, the results demonstrate that trajectory-aware BAI effectively allocates evaluations among multiple TRs and improves BO performance on practical multimodal optimization problems.
Moreover, since TuRBO-$m$-BAI consistently outperforms GP-EI, GP-TS, and TuRBO-1, the effectiveness of our theoretical analysis is further supported by experimental evidence.

\section{Extension to Non-BO Optimizers}
\label{sec:experiment_ea}

Although the main focus of this paper is improving BO, the proposed trajectory-aware BAI mechanism only requires optimization histories.
It is therefore not tied to GP regression, acquisition functions, or TRs.
To examine this generality, we apply the same BAI procedure to non-BO optimizers by running multiple independently initialized optimizer instances in parallel, treating each instance as an arm, and using trajectory prediction with sequential halving to eliminate low-potential instances.

\subsection{Setup}
We apply the proposed BAI mechanism to three evolutionary algorithms: covariance matrix adaptation evolution strategy (CMAES) \cite{hansen2001completely}, particle swarm optimization (PSO) \cite{kennedy1995particle}, and differential evolution (DE) \cite{storn1997differential}.
We evaluate the resulting CMAES-$m$-BAI, PSO-$m$-BAI, and DE-$m$-BAI on the Robot Pushing and Rover Trajectory.
We compare CMAES-$m$-BAI with CMAES and BIPOP-CMAES \cite{hansen2009benchmarking}, PSO-$m$-BAI with standard PSO, and DE-$m$-BAI with standard DE.
For all three BAI variants, we use $m=10$ parallel optimizer instances. The population size of PSO and DE is 20. For CMAES, we use the recommended default population size: $4 + \lfloor 3 \cdot \log D \rfloor$.
The total evaluation budget is set to $N=20000$ for each task.
This larger budget reflects the lower sample efficiency but lower per-iteration overhead of these optimizers.
All methods are evaluated over 31 independent trials.

\begin{figure}[t]
    \centering
    \includegraphics[width=0.95\columnwidth]{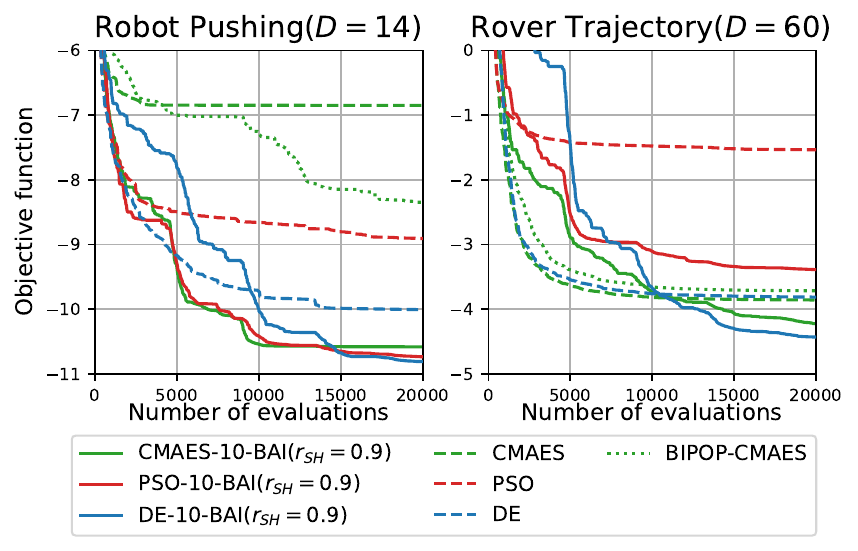}
    \caption{Comparison of trajectory-aware BAI applied to CMAES, PSO, and DE on two real-world problems.}
    \label{fig:real_ea}
\end{figure}

\subsection{Results}
Figure \ref{fig:real_ea} shows the results on the real-world benchmarks.
Overall, the BAI-based variants improve upon their corresponding base optimizers in both problems. CMAES-10-BAI also outperforms BIPOP-CMAES.
These results indicate that the benefit of trajectory-aware BAI is not limited to TuRBO.
Although our theoretical analysis focuses on BO with multiple TRs, the empirical results show that the same optimization-history-based mechanism can also improve other optimizers.

\section{Conclusions}
In this work, we proposed a trajectory-aware BAI framework that allocates evaluations among parallel search processes using optimization history for multimodal black-box optimization. Integrated with TuRBO, the method theoretically accelerates convergence to the global optimum compared with standard BO based on a single global GP model. Experiments on synthetic and real-world benchmarks showed that it improves TuRBO, especially on challenging multimodal problems. We also showed that the same BAI mechanism enhances other evolutionary algorithms, suggesting its broad applicability to multimodal black-box optimization.

\bibliography{mybib}


\newpage
\appendix
\onecolumn

\setcounter{secnumdepth}{2}
\section{Detailed Theoretical Assumptions and Proofs}
\newcommand{\rsh}{r_{\mathrm{SH}}}

In this section, we first formally state the assumptions that were deferred from Section \ref{sec:theory} for conciseness (Section \ref{sec:assumptions}).
We then prove Proposition \ref{prop:norm-comparison} and Theorem \ref{thm:simple-regret-bound} in Section \ref{apx:proofs}.
Since Theorem \ref{thm:simple-regret-bound} assumes $\nsh = \Theta(\log N)$ whereas the best empirical performance is obtained at $\nsh = \Theta(N)$, we provide a finite-budget analysis covering the whole range of $\nsh$, followed by a discussion of what our results do and do not establish about the proposed method.


\label{apx:theory}
\subsection{Formal Theoretical Assumptions}
\label{sec:assumptions}
First, we state assumptions required for Proposition \ref{prop:norm-comparison}.
\begin{assumption}
    \label{assumption:appendix-prop}
    For each $1 \le j \le \numtr$, we assume $f^{(j)}$ belongs to the RKHS $\cH_{k^{(j)}}(\trr^{(j)})$ with $k^{(j)} := k_{\alpha^{(j)}}$, i.e.,
    the length scale parameter of $f^{(j)}$ is $\alpha^{(j)} > 0$.
    We assume that $f$ belongs to $\cH_{k_{\alpha}}(\Omega)$ and $\alpha \le \min_{1 \le j \le \numtr} \alpha^{(j)}$.
\end{assumption}
\begin{remark}
        We remark that Assumption \ref{assumption:appendix-prop} is valid 
        since we consider a Mat\'ern kernel with a fixed smoothness parameter. 
        More formally, since $f^{(j)}$ is a restriction of $f \in \cH_{k_{\alpha}}(\Omega)$ to $\trr^{(j)}$,
        by \cite[Theorems 10.46, 10.47]{wendland2004scattered}, 
        there exists a unique extension $\widetilde{f}^{(j)} \in \cH_{k_{\alpha}}(\RR^D)$ of $f^{(j)}$.
        By \cite[Theorem 1]{larsson2024scaling}, for any length scale parameter $\alpha^{(j)}$,
        $f^{(j)}$ belongs to $\cH_{k_{\alpha^{(j)}}}(\RR^D)$, which is denoted as $\cH_{k^{(j)}}(\RR^D)$.
        By \cite[Theorem 10.47]{wendland2004scattered}, $f^{(j)}$ belongs to $\cH_{k^{(j)}}(\trr^{(j)})$.
\end{remark}
Besides Assumption \ref{assumption:appendix-prop}, we assume the following for Theorem \ref{thm:simple-regret-bound}.
\begin{assumption}
    \label{assumption:appendix}
    \begin{enumerate}
        \item For each $1 \le j \le \numtr$, we assume 
$\simplereg_N(\cA^{(j)}; \trr^{(j)}, f^{(j)}) = \widetilde{O}(U_N(k^{(j)}, \trr^{(j)}, \|f^{(j)}\|_{k^{(j)}}))$
        with $\hat{\bfx}_{*}^{(i)} = \bfx^{(i)}$ for $1 \le i \le N$.
        \item For $1 \le j \le \numtr$ and $i \ge 1$,
        the expected simple regret is modeled as 
$\simplereg_i(\cA^{(j)}; \trr^{(j)}, f^{(j)}) = \sum_{l=1}^{d} a_l (1 + i)^{-\rho_l}$
        with $0 \le \rho_l < 1/2$ and $a_l \in \RR$ for $1 \le l \le d$, where $\rho_1,\dots, \rho_d$ are pairwise distinct.
        \item We assume there exists a constant $C> 0$ such that 
        $\|f^{(j)}\|_{k^{(j)}} + R \le C \|f^{(j)}\|_{k^{(j)}}$ for each $1 \le j \le \numtr$,
        and $(\|f\|_{k_{\alpha}} + R) \le C \|f\|_{k_{\alpha}}$,
        where $R$ is the scale of the noise.
        \item For each $j, j' \in [\numtr]$ with $j \ne j'$,
        we assume $\min_{\bfx \in \trr^{(j)}}f^{(j)}(\bfx) \neq \min_{\bfx \in \trr^{(j')}}f^{(j')}(\bfx)$,
        that is, the objective function $f$ has distinct minimum values on trust regions $\trr^{(j)}$.
        \item For the parameters of the BAI algorithm, we take sufficiently large $\tilde{T}$
        and $\nsh = O(\beta \log N)$, where $\beta > 1/2$ is a constant.
        We also assume that $\lfloor \frac{\nsh}{\numtr \lceil \log_2 \numtr \rceil}\rfloor \ge d$.
    \end{enumerate}
\end{assumption}

We explain each condition in Assumption \ref{assumption:appendix} as follows:
The first condition implies that each BO algorithm $\cA^{(j)}$ operating on $\trr^{(j)}$ is a reasonable algorithm
that achieves a nearly optimal simple regret bound.
For instance, Phased Elimination \cite{li2022gaussian,lattimore2020learning} achieves this simple regret bound
with $\hat{\bfx}_{*}^{(i)} = \bfx^{(i)}$ for $1 \le i \le N$,
where $\hat{\bfx}_{*}^{(i)}$ is the point used in the definition of the simple regret in Eq. \eqref{eq:simple-reg-bound}.
We refer to Proposition \ref{prop:pe}.
For the sake of self-containedness,
we provide Phased Elimination \cite{li2022gaussian,lattimore2020learning} applied to GP bandits in Algorithm \ref{alg:pe}.
The second condition is the loss model assumption, which is equivalent to Eq. \eqref{eq:loss},
and the same condition was assumed in \cite{takemori24bai}.
In Section \ref{sec:app-validation-of-trajectory-pred}, we empirically show the validity of this assumption in our setting.
Combined with the first condition, it implies that the loss sequence observed by SH on arm $j$ decomposes as
\begin{equation}
    \label{eq:loss-decomposition}
    \mathbb{E}\left[y^{(i, j)}\right]
    = f_*^{(j)} + \simplereg_i(\cA^{(j)}; \trr^{(j)}, f^{(j)}),
    \qquad f_*^{(j)} := \min_{\bfx \in \trr^{(j)}} f(\bfx),
\end{equation}
where the first term is independent of $i$ and the second vanishes as $i \to \infty$.
Accordingly, the constant basis function of the model Eq. \eqref{eq:loss} corresponds to
$f_*^{(j)}$ and the decaying ones to the simple regret of $\cA^{(j)}$ within $\trr^{(j)}$.
The third condition implies that the complexity of the objective function $\|f\|_{k_\alpha}$ is dominant compared to the scale $R$ of the noise.
The fourth condition thus guarantees that the limits $f_*^{(j)}$ of the $\numtr$ loss sequences are
pairwise distinct, which is what makes the best arm identifiable.
The fifth condition determines the parameters of the BAI-based algorithm, and the inequality
 $\lfloor \frac{\nsh}{\numtr \lceil \log_2 \numtr \rceil}\rfloor \ge d$ is necessary to apply an existing result on the error probability 
 of misidentification \cite[Theorem 4.4]{takemori24bai}.

\begin{remark}
        We note that although we assume 
        $\Omega = \cup_{1 \le j \le \numtr}\trr^{(j)}$, our analysis holds even if $\Omega \supseteq \cup_{1 \le j \le \numtr}\trr^{(j)}$ with $\cup_{1 \le j \le \numtr}\trr^{(j)} \ni \bfx_{*}$.
\end{remark}

\begin{algorithm}[t]
\caption{Phased Elimination (PE) for GP Bandits (finite $\mathcal{X}$)}
\label{alg:pe}
\textbf{Input}: domain $\mathcal{X}$, horizon $N$,
confidence $\delta$, RKHS-norm bound $\Psi$, regularizer $\reggp$,
sub-Gaussian parameter $R$ \\
\begin{algorithmic}[1]
\STATE $\mathcal{X}^{(1)} \leftarrow \mathcal{X}$;\quad $i \leftarrow 1$;\quad $N^{(1)} \leftarrow 1$
\STATE $\sqrt{\beta} \leftarrow \Psi + \dfrac{R}{\reggp}\sqrt{2\log\!\big(|\mathcal{X}|(\lceil\log_2 N\rceil+1)/\delta\big)}$
\FOR{phase $k = 1, 2, \dots$}
    \IF{$i > N$}
        \STATE \textbf{break}
    \ENDIF
    \STATE $S^{(k)} \leftarrow \emptyset$
    \FOR{$l = 1, 2, \dots, N^{(k)}$}
        \STATE $\bfx^{(i)} \leftarrow \arg\max_{\bfx \in\mathcal{X}^{(k)}}\hat{\sigma}^{2}(\bfx; S^{(k)})$
            \hfill {\footnotesize$\triangleright$ pure exploration; variance is observation-free}
        \STATE Select $\bfx^{(i)}$;\quad $S^{(k)} \leftarrow S^{(k)} \cup \{\bfx^{(i)}\}$;\quad $i \leftarrow i+1$
        \IF{$i > N$}
            \STATE \textbf{break}
        \ENDIF
    \ENDFOR
    \STATE Observe objective function values $\{y^{(i)} : \bfx^{(i)}\in S^{(k)}\}$
    \STATE $\mathrm{UCB}^{(k)}(\cdot)\leftarrow \hat{f}(\cdot; S^{(k)})+\sqrt{\beta}\,\hat{\sigma}(\cdot; S^{(k)})$;\quad
           $\mathrm{LCB}^{(k)}(\cdot)\leftarrow \hat{f}(\cdot; S^{(k)})-\sqrt{\beta}\,\hat{\sigma}(\cdot; S^{(k)})$
    \STATE $\mathcal{X}^{(k+1)}\leftarrow\big\{\bfx\in\mathcal{X}^{(k)}:\ \mathrm{LCB}^{(k)}(\bfx)\le \min_{\bfx'\in\mathcal{X}^{(k)}}
    \mathrm{UCB}^{(k)}(\bfx')\big\}$
        \hfill {\footnotesize$\triangleright$ keep potential minimizers}
    \STATE $N^{(k+1)} \leftarrow 2 N^{(k)}$
\ENDFOR
\end{algorithmic}
\end{algorithm}

\citet{li2022gaussian} proved that the batched pure exploration (BPE) algorithm achieves a near-optimal
cumulative regret bound.
Their proof also implies that PE achieves a near-optimal simple regret bound.
For completeness, we provide the proof below.
\begin{proposition}
    \label{prop:pe}
    Suppose that Assumption 4 of \cite{vakili2021optimal} holds for the objective
    function $f: \Omega \rightarrow \RR$; this is a standard assumption satisfied by
    typical kernels such as the SE and Mat\'ern kernels \cite{vakili2021optimal}.
    Then there exists a discretization $D_{N}$ of $\Omega$ with $|D_{N}| = O(N^{D/2})$.
    We apply PE (Algorithm~\ref{alg:pe}) to the objective function $f$ over the
    domain $\cX=D_{N}$, and for $i = 1, \dots, N$ we denote by $\bfx^{(i)}$ the point
    selected in the $i$-th round by PE.
    Furthermore, we assume that $\reggp = \Theta(1)$.
    Then, with probability at least $1 - \delta$, PE achieves the following
    high-probability simple regret bound:
    \begin{equation*}
        f(\bfx^{(i)}) - f(\bfx_*)
        = \widetilde{O}\!\left(
            \sqrt{\frac{\gamma_i(k, \Omega)}{i}}\,\bigl(\|f\|_k + R\bigr)
        \right),
    \end{equation*}
    where $\bfx_* \in \arg\min_{\bfx \in \Omega} f(\bfx)$, 
    and $\gamma_i = \gamma_i(k, \Omega)$ is the maximal information gain,
defined as $\sup_{\bfz^{(1)}, \dots, \bfz^{(i)} \in \Omega} \frac{1}{2} \log \det (I_i +\reggp^{-2} k(\bfz^{1\dots i}, \bfz^{1\dots i}))$,
where $\reggp> 0$ is a regularizer, $k(\bfz^{1\dots i}, \bfz^{1\dots i}) = 
(k(\bfz^{(s)},\bfz^{(s')}))_{1\le s, s' \le i} \in \RR^{i\times i}$,
    and
    the notation $\widetilde{O}$
    absorbs $\mathrm{poly}\bigl(\log N,\, \log(1 + \|f\|_k),\, \log(1/\delta)\bigr)$
    factors, and $R$ is the noise scale.
\end{proposition}
\begin{proof}
   By Assumption 4 of \cite{vakili2021optimal},
   restricting the domain $\Omega$ to its discretization $D_{N}$ incurs an error bounded by $O(1/\sqrt{N})$.
   Therefore, 
   it is enough to prove that PE achieves the simple regret bound for the restricted domain $D_{N}$.
   Let $1 \le i \le N$ and assume that $i$ is in the $k$-th phase in Algorithm \ref{alg:pe}.
   Then, by Lemma 4 of \cite{li2022gaussian}, 
   we have the following (on a good event where the confidence interval holds):
   \begin{math}
    r_i = O \left(
        \sqrt{\frac{\gamma_{N^{(k-1)}}\beta}{N^{(k-1)}}}
     \right).
   \end{math}
   Here, $N^{(k)} = 2^{k-1}$ is defined in Algorithm \ref{alg:pe}.
   Since $N^{(k-1)} \le i$ and $i = \Theta (N^{(k-1)})$,
   we have
   \begin{math}
    r_i = O \left(
        \sqrt{\frac{\gamma_{i}\beta}{i}}
     \right).
   \end{math}
   Noting that $|D_N|$ is of polynomial order in $N$,
   we have our assertion. 
\end{proof}

\subsection{Proof of the Main Theoretical Results}
\label{apx:proofs}

We prove Proposition \ref{prop:norm-comparison} as follows.
\begin{proof}[Proof of Proposition \ref{prop:norm-comparison}]
    We denote $\widetilde{f} \in \cH_{k_\alpha}(\RR^D)$ by the unique extension of $f \in \cH_{k_\alpha}(\Omega)$
    defined by \cite[Theorem 10.46]{wendland2004scattered},
    that is, $\widetilde{f}|_{\Omega} = f$ and $\|\widetilde{f}\|_{k_{\alpha}} = \|f\|_{k_{\alpha}}$.
    Similarly, we define $\widetilde{f}^{(j)} \in \cH_{k^{(j)}}(\RR^D)$ by the unique extension of 
    $f^{(j)} \in \cH_{k^{(j)}}(\trr^{(j)})$.
    Then, we have the following:
    \begin{align*}
        &\|f^{(j)}\|_{k^{(j)}}
        = \|\widetilde{f}^{(j)}\|_{k^{(j)}}
        \le \|\widetilde{f}\|_{k^{(j)}}\\
        &\lesssim 
        \left(\frac{\alpha}{\alpha^{(j)}}\right)^{D/2}
         \|\widetilde{f}\|_{k_\alpha}
        =\left(\frac{\alpha}{\alpha^{(j)}}\right)^{D/2}
         \|f\|_{k_\alpha}.
    \end{align*}
    Here, the first equality follows from \cite[Theorem 10.46]{wendland2004scattered},
    the first inequality follows from \cite[Theorem 10.47]{wendland2004scattered} 
    since both $\widetilde{f}^{(j)}$ and $\widetilde{f}$ are extensions of $f^{(j)}$,
    the second inequality follows from \cite[Theorem 1]{larsson2024scaling},
    and the last equality follows from \cite[Theorem 10.46]{wendland2004scattered}.
\end{proof}

\begin{lemma}
    \label{lem:gamma-comparison}
    Let $\Omega = [0, 1]^D$.
    Then, for any $j \in [\numtr]$, we have $\gamma_N(k^{(j)}, \trr^{(j)}) \le \gamma_N(k_\alpha, \Omega)$.
\end{lemma}
\begin{proof}
    First, we note that $k^{(j)} = k_{\alpha^{(j)}}$ and $\alpha^{(j)} \ge \alpha$ by our assumption.
    For $\bfz^{(1)}, \dots, \bfz^{(N)} \in \RR^D$ and a positive definite kernel $k$,
    we define $k(\bfz^{1\dots N}, \bfz^{1\dots N}) \in \RR^{N \times N}$ so that the $(a, b)$-th entry is given as $k(\bfz^{(a)}, \bfz^{(b)})$,
    where $1 \le a, b \le N$.
    By definition of the maximum information gain, we see that
    \begin{align*}
       &\gamma_N(k^{(j)}, \trr^{(j)}) \\&=
       \sup_{\bfz^{(1)}, \dots, \bfz^{(N)} \in \trr^{(j)}}
         \frac{1}{2} \log \det (I_N +\tilde{\lambda}^{-2} k^{(j)}(\bfz^{1\dots N}, \bfz^{1\dots N}))\\
         &\le  \sup_{\bfz^{(1)}, \dots, \bfz^{(N)} \in \Omega}
         \frac{1}{2} \log \det (I_N +\tilde{\lambda}^{-2} k^{(j)}(\bfz^{1\dots N}, \bfz^{1\dots N}))\\
         & = 
                \sup_{\bfz^{(1)}, \dots, \bfz^{(N)} \in \frac{\alpha}{\alpha^{(j)}} \Omega}
         \frac{1}{2} \log \det (I_N +\tilde{\lambda}^{-2} k_\alpha(\bfz^{1\dots N}, \bfz^{1\dots N}))\\
         & \le 
                \gamma_N(k_\alpha, \Omega).
    \end{align*}
    The last inequality holds since $\alpha / \alpha^{(j)} \in [0, 1]$ and $\frac{\alpha}{\alpha^{(j)}} \Omega \subseteq \Omega$.
\end{proof}

Now, we prove Theorem \ref{thm:simple-regret-bound} as follows.
\begin{proof}[Proof of Theorem \ref{thm:simple-regret-bound}]
    By Assumption \ref{assumption:appendix-prop}, Assumption \ref{assumption:appendix},
    and \cite[Theorem 4.4]{takemori24bai}, the probability of misidentifying the best trust region $\trr^{(j_*)}$ is bounded as 
    $O\left(\exp\left(- \Omega\left(\nsh\right) \right)\right)$.
    Therefore, with $\nsh = O\left(\beta \log N\right)$ and $\beta > 1/2$,
    we have the following:
    \begin{align*}
        &\simplereg_N(\algours; \Omega, f)\\
         &\le O(N^{-\beta}) +
        \simplereg_{N -\nsh}(\cA^{(j_*)}; \trr^{(j_*)}, f^{(j_*)})\\
        &=
        \widetilde{O}\left(\sqrt{\frac{\gamma_{N-\nsh}(k^{(j_*)}, \trr^{(j_*)})}{N-\nsh}} \|f^{(j_*)}\|_{k^{(j_*)}} \right).
    \end{align*}
    Here, the last equality holds 
    by the third condition of Assumption \ref{assumption:appendix} and 
    the fact that the term $O(N^{-\beta})$ is not dominant because $\beta > 1/2$.
    Then, noting Assumption \ref{assumption:appendix},
    Lemma \ref{lem:gamma-comparison}, and Proposition \ref{prop:norm-comparison},
    we have our assertion of the theorem.
\end{proof}

\subsection{Finite-budget analysis of the BAI budget}
\label{apx:finite-budget}
Theorem \ref{thm:simple-regret-bound} assumes $\nsh = \Theta(\log N)$, whereas
our experiments in Appendix \ref{apx:proportion} show that the best performance is achieved at $\rsh \in [0.9, 1.0]$, i.e.\ at $\nsh = \Theta(N)$.
The two are reconciled by the observation that SH does not discard the evaluations spent on
surviving arms: the arm it returns has itself been pulled $\Omega(\nsh / \log_2 \numtr)$ times, and
the sample set $\cD^{(j)}$ of each trust region is retained across rounds in Algorithm
\ref{alg:nsbai}.
For the Mat\'ern kernel we show below that devoting the entire budget to arm identification costs
only a $\mathrm{polylog}(\numtr)$ factor and leaves the rate in $N$ unchanged.

Throughout, $L := \lceil \log_2 \numtr \rceil$, the surviving arm sets of SH are
$A_1 \supsetneq \cdots \supsetneq A_{L+1}$ with $|A_r| = \lceil \numtr / 2^{r-1} \rceil$, and
$\rsh := \nsh / N = (\nsh + n \numtr)/N$ is the fraction of the budget spent on SH
(we do not consider the initial budget $n \numtr$ in our analysis).
Let $j_*$ be the index with $\bfx_* \in \trr^{(j_*)}$, so that $f_*^{(j_*)} = f(\bfx_*)$ in the
notation of Eq. \eqref{eq:loss-decomposition}, let $\Delta_j := f_*^{(j)} - f(\bfx_*) > 0$ for
$j \ne j_*$ be the gaps, positive by the fourth condition of Assumption \ref{assumption:appendix},
let $\Delta_{(2)} \le \cdots \le \Delta_{(\numtr)}$ be their increasing rearrangement, and set
$H_2 := \max_{2 \le j \le \numtr} j / \Delta_{(j)}^2$.
Then \citet[Theorem 4.4]{takemori24bai} gives
\begin{equation}
    \label{eq:misid}
    \Pr[\hat{\jmath} \ne j_*]
    = O\Big(
        \numtr \log \numtr
        \exp\big(- \Omega(\nsh / (\gamma_{\boldsymbol{\rho}} H_2 \log \numtr))\big)
    \Big),
\end{equation}
with $\gamma_{\boldsymbol{\rho}} > 0$ depending only on $(\rho_1, \dots, \rho_d)$.

The SH phase is not overhead for the region it selects: sequential halving spends $\nsh / L$ in each
of its $L$ rounds but on ever fewer arms, and only two remain in the last, so the survivor absorbs
$\nsh / (2L)$ on its own.
Evaluations inside $\trr^{(j_*)}$ therefore accumulate whether they are spent before or after
selection, and the worst case over $\rsh$ is $\rsh = 1$, where $N / (2L)$ replaces the $N - \nsh$
rounds allotted to $\cA^{(j_*)}$ in Theorem \ref{thm:simple-regret-bound}.
The factor $L = \lceil \log_2 \numtr \rceil$ is thus the price of halving $\numtr$ arms, not of
devoting the budget to identification.

\begin{proposition}
    \label{prop:finite-budget}
    Suppose Assumptions \ref{assumption:appendix-prop} and \ref{assumption:appendix}.    Then,  for
    every $\rsh \in (0, 1]$,
    \begin{align}
        \label{eq:finite-budget}
        \simplereg_N(\algours; \Omega, f)
        \ \le\ &
        2\|f\|_{k_\alpha} \cdot O\Big(
            \numtr \log \numtr
            \exp\big(-\Omega(\rsh N / (\gamma_{\boldsymbol{\rho}} H_2 \log \numtr))\big)
        \Big)
        \nonumber \\
        & + \widetilde{O}\left(
            \left(\frac{\alpha}{\alpha^{(j_*)}}\right)^{D/2}
            U_{N/(2L)}(k_\alpha, \Omega, \|f\|_{k_\alpha})
        \right),
    \end{align}
    whose two summands we call the misidentification and the optimization term.
\end{proposition}
 
\begin{proof}
    The argument is that of Theorem \ref{thm:simple-regret-bound} with two changes.
    On the misidentification event we bound the regret by $2\|f\|_{k_\alpha}$ instead of
    $O(N^{-\beta})$, using
    $|f(\bfx)| \le \|f\|_{k_\alpha} \sqrt{k_\alpha(\bfx, \bfx)} = \|f\|_{k_\alpha}$, and multiply by
    Eq. \eqref{eq:misid} with $\nsh = \rsh N$.
    On its complement we replace the $N - \nsh$ rounds allotted to $\cA^{(j_*)}$ by $N / (2L)$: each
    arm of $A_r$ receives $\lfloor \nsh / (|A_r| L) \rfloor$ evaluations in round $r$, whence the
    $O(\numtr L)$ rounding term, and $j_*$ lies in every $A_r$, so the final round alone, where
    $|A_L| = 2$, contributes $\rsh N / (2L)$ pulls; adding the post-selection budget $(1 - \rsh)N$
    gives $\frac{\rsh N}{2L} + (1 - \rsh)N \ge \frac{N}{2L}$, the minimum being at $\rsh = 1$.
    These evaluations accumulate because $\cD^{(j)}$ is retained and $\cA^{(j)}$ depends only on its
    own history.
    The remaining steps are unchanged.
\end{proof}

For the Mat\'ern-$\nu$ kernel, $\gamma_N(k_\alpha, \Omega) = \widetilde{O}(N^{D/(2\nu+D)})$
\citep{srinivas2010gaussian,vakili2021optimal}, so
$U_N(k_\alpha, \Omega, \|f\|_{k_\alpha})
= \widetilde{O}(N^{-\nu/(2\nu+D)}(\|f\|_{k_\alpha} + R))$.
Since $N / (2L) = \Theta(N / \log_2 \numtr)$, the optimization term of Eq.
\eqref{eq:finite-budget} is thus governed by the ratio $\log_2 \numtr / N$ in place of $1/N$.

\begin{corollary}
    \label{cor:matern-explicit}
    Under the assumptions of Proposition \ref{prop:finite-budget} with a Mat\'ern-$\nu$ kernel, for
    every $\rsh \in (0, 1]$,
    \begin{equation}
        \label{eq:matern-explicit}
        \simplereg_N(\algours; \Omega, f)
        = \widetilde{O}\left(
            \left(\frac{\log_2 \numtr}{N}\right)^{\frac{\nu}{2\nu + D}}
            \left(\frac{\alpha}{\alpha^{(j_*)}}\right)^{D/2}
            (\|f\|_{k_\alpha} + R)
        \right)
        + 2\|f\|_{k_\alpha} O\Big(
            \numtr \log \numtr\,
            e^{-\Omega(\rsh N / (\gamma_{\boldsymbol{\rho}} H_2 \log \numtr))}
        \Big).
    \end{equation}
    Moreover, if $\rsh N = \Omega(\gamma_{\boldsymbol{\rho}} H_2 \log \numtr \cdot \log N)$, so
    that the second summand is dominated by the first, then the bound improves on
    $U_N(k_\alpha, \Omega, \|f\|_{k_\alpha})$ as soon as
    \begin{equation}
        \label{eq:threshold}
        \frac{\alpha^{(j_*)}}{\alpha} > (2L)^{\frac{2\nu}{D(2\nu + D)}} .
    \end{equation}
\end{corollary}

\begin{proof}
    The first claim follows by substituting
    $U_{N/(2L)}(k_\alpha, \Omega, \|f\|_{k_\alpha})
    = \widetilde{O}((2L/N)^{\nu/(2\nu+D)} (\|f\|_{k_\alpha} + R))$ into Eq.
    \eqref{eq:finite-budget} and absorbing the constant $2^{\nu/(2\nu+D)} \le 2$.
    For the second, the first summand equals
    $O((2L)^{\nu/(2\nu+D)}) (\alpha/\alpha^{(j_*)})^{D/2}$ times $U_N$, and requiring this factor to
    be below one gives Eq. \eqref{eq:threshold}; the condition on $\rsh$ makes the exponent in Eq.
    \eqref{eq:misid} exceed $\frac{\nu}{2\nu + D} \log N$, which is at most $\frac{1}{2} \log N$.
\end{proof}

Since $2\nu / (D(2\nu + D)) = O(D^{-2})$, the right-hand side of Eq. \eqref{eq:threshold} tends to
one rapidly with the dimension: the length scale around $\bfx_*$ need exceed the global one only by a margin
vanishing in $D$, consistently with Fig. \ref{fig:example-diffrent-ls} and with the gains of the
proposed method being largest on strongly multimodal, high-dimensional problems.

\begin{remark}
    \label{rem:why-large-rsh}
    Eq. \eqref{eq:matern-explicit} exhibits the two effects of $\nsh$: once the effective budget
    $N / (2L)$ established in the proof of Proposition \ref{prop:finite-budget} is used, the first
    summand is independent of $\rsh$ while the second decreases exponentially in it.
    The bound is therefore never penalized by taking $\rsh$ as large as possible, and is strictly
    improved by doing so whenever
    $\gamma_{\boldsymbol{\rho}} H_2 \log \numtr = \Omega(N / \log N)$, i.e.\ whenever the
    misidentification term is not already negligible at smaller $\rsh$.
    This matches our experimental results, where $\rsh \in [0.9, 1.0]$ performs best with
    little difference between the endpoints while smaller $\rsh$ leads to convergence to local
    optima, and $H_2$ is large precisely on the strongly multimodal problems, where the trust-region
    minima $f_*^{(j)}$ are closely spaced.
\end{remark}

\subsection{Scope of the Theoretical Analysis}
Theorem \ref{thm:simple-regret-bound} and Proposition \ref{prop:finite-budget} cover three aspects
of the proposed method.
Restricting the objective to a trust region reduces both quantities that control the regret of a
nearly optimal BO algorithm, the RKHS norm (Proposition \ref{prop:norm-comparison}) and the maximal
information gain (Lemma \ref{lem:gamma-comparison}), by a factor exponential in $D$.
The selection must be adaptive, since for any index fixed independently of the observations there are
objectives satisfying our assumptions with $\bfx_* \notin \trr^{(j)}$, for which the simple regret
does not vanish.
And the evaluations spent on identification are not lost, so the budget $\rsh \in [0.9, 1.0]$ used in
our experiments is consistent with the analysis rather than in tension with it.
The analysis also accounts for a design choice of Algorithm \ref{alg:nsbai}: because the arms are
the fixed regions $\trr^{(j)}$, restarts have to be suppressed for the identification to refer to a
well-defined set of arms, which is what the constraint $l^{(j)} \ge l_{\min}$ enforces.
 
The following features of the implemented algorithm lie outside the theoretical analysis.
(i) The regions are static, with $\Omega = \cup_j \trr^{(j)}$ fixed in advance, whereas Algorithm
\ref{alg:nsbai} draws them from random samples and then moves each centre $c^{(j)}$ and rescales
each side length $l^{(j)}$ by TuRBO's success and failure rules.
(ii) The base optimizer is assumed to attain the rate Eq. \eqref{eq:simple-reg-bound}, as Phased
Elimination does but TuRBO's Thompson sampling is not known to; the results therefore concern
BAI-guided allocation over trust regions for a nearly optimal local BO algorithm rather than TuRBO
itself.
(iii) The ridge regression is part of the analysis, but its input is not: Section \ref{sec:theory}
feeds it the raw losses $y^{(i)}$, whereas the implementation feeds the preprocessed history
$\cD_*^{(j)}$.
(iv) The initial budget $n \numtr$ is not considered in our analysis.


\section{Implementation Details}
\label{apx:setting}
\paragraph{Common Setting for Trajectory-aware BAI for TuRBO, CMAES, PSO, and DE}
In the main text, we use $r_{\mathrm{SH}}=0.9$, where $r_{\mathrm{SH}}$ denotes the fraction of the total budget allocated to initialization and sequential halving.
Results for other values of $r_{\mathrm{SH}}$ used in TuRBO-$m$-BAI are provided in Appendix \ref{apx:proportion}.
For the optimization trajectory prediction used in BAI, we set the number of parameters in Eq. \eqref{eq:loss} to $d = 50$, and use ridge regression with a regularization parameter $\lambda = 0.1$ according to the experimental results of the original paper. 

\paragraph{Common Setting for Bayesian Optimization}
In our implementations of TuRBO, GP-EI, and GP-TS with \texttt{Botorch} \cite{balandat2020botorch}, all methods use GP with a Matérn-5/2 kernel and automatic relevance determination. Initial samples are generated using Sobol sequences, and the random seed is set based on the trial index ranging from 1 to 31. For DSP, CobBO, and Bounce, configurations follow the settings provided in their respective source code repositories.

\paragraph{TuRBO and GP-TS}
As for the TuRBO parameters, we follow the original paper and use $l_0 = 0.8$, $l_{\min} = 0.5^7$, and $\tau_{\text{succ}} = 3$. To prevent immediate shrinkage of the TR upon a single failed step, we impose a lower bound on $\tau_{\text{fail}}$ as $\tau_{\text{fail}} = \max(D/b, 2)$. 
TuRBO also includes a parameter that limits the maximum expansion of $l$, which is fixed at $l_{\text{max}} = 1.6$. 
The implementation of TuRBO used in our experiments is based on the \texttt{Botorch}.

 For TS in TuRBO and GP-TS, we employ the random axis-aligned subspace perturbation strategy \cite{rashidi2024cylindrical}, as in the original TuRBO. In this method, 2,000 to 5,000 candidate points (depending on the problem dimension $D$) are sampled within the TR by randomly substituting a subset of the TR center $\mathbf{c}^{(j)}$'s variables (TuRBO) or the best sample (GP-TS) with values from a Sobol sequence. 

\paragraph{GP-EI}
LogEI is miximized via the L-BFGS-B algorithm, using 512 initial points and 10 restarts, consistent with the default configuration in \texttt{Botorch}.

\paragraph{DSP}
Based on the source code accompanying the original paper, we use the Gaussian (RBF) kernel together with the qLogNEI acquisition function, a noise-aware version of qLogEI.

\paragraph{CobBO}
For all problems, we use the same implementation and parameters provided in the original paper without modification. 

\paragraph{Bounce}
For all tasks, the optimization process begins in a 2-dimensional subspace, with up to $N/4$ samples evaluated before transitioning to the full-dimensional space of the problem.

\paragraph{CMAES and BIPOP-CMAES}
We used the implementation of \texttt{pycma} \cite{hansen2019pycma} with its default settings. For CMAES, the default population size is $\lambda = 4 + \lfloor 3 \log D \rfloor$.

\paragraph{PSO and DE}
We used the implementations of PSO and DE provided by \texttt{pymoo} \cite{pymoo}. For both algorithms, the population size was set to 20. 
For PSO, all other hyperparameters were set to the default values of \texttt{pymoo}: $w=0.9$, $c_1=2.0$, $c_2=2.0$, \texttt{adaptive=True}, \texttt{initial\_velocity=random}, \texttt{max\_velocity\_rate=0.20}, and \texttt{pertube\_best=True}. 
For DE, we used the \texttt{DE/rand/1/bin} variant with \texttt{CR=0.9}, \texttt{dither=vector}, and \texttt{jitter=False}; the other hyperparameters were kept at their \texttt{pymoo} default values.

\section{Benchmark Problems}
As synthetic test functions, we use the 10-dimensional Schwefel, Rastrigin, Ackley, and Rosenbrock functions, with their respective domains set to $[-500, 500]^{10}$, $[-3, 4]^{10}$, $[-5, 10]^{10}$, and $[-5, 10]^{10}$. Schwefel and Rastrigin are highly multimodal and are used to evaluate the effectiveness of the proposed method in multimodal settings, whereas Ackley and Rosenbrock, which are weakly multimodal or unimodal, are used to assess its performance on less favorable problem classes.

For real-world problems, we use the 14-dimensional Robot Pushing problem and the 60-dimensional Rover Trajectory problem, both treated as minimization problems by negating the original objective values \cite{wang2018batched}. In both cases, many optimization trials converge without reaching the global optimum (see Appendix \ref{apx:detailed_results}), suggesting the presence of complex, multimodal landscapes with prominent local optima.

\section{Details of Main Results: Final Objective Values and Statistical Testing}
\label{apx:detailed_results}

\begin{figure}[t]
    \centering
    \includegraphics[width=0.95\columnwidth]{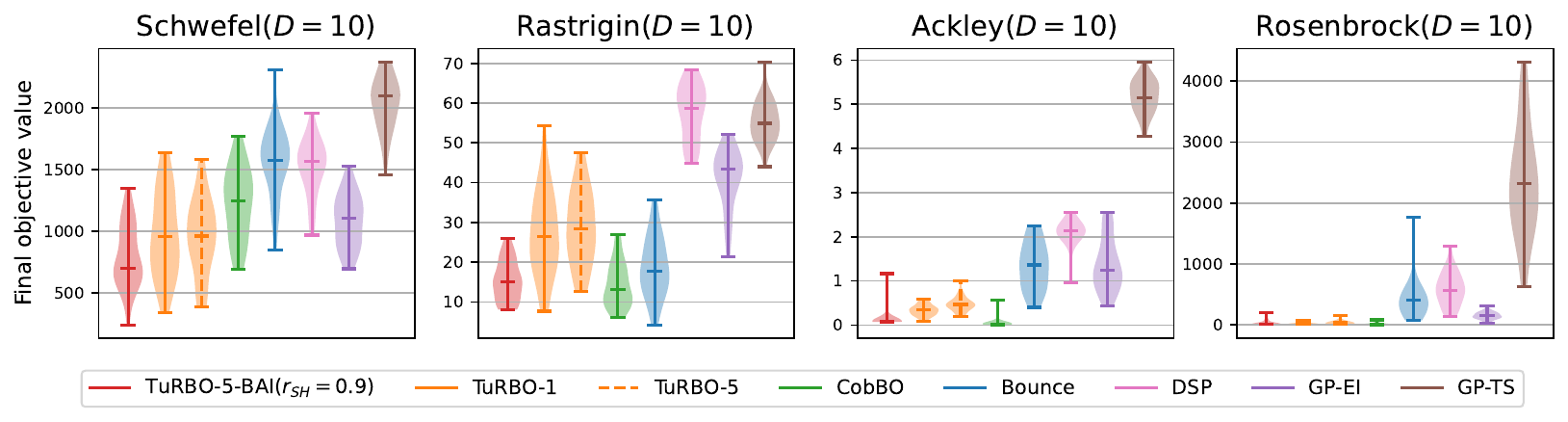}
    \caption{Distribution of the best objective values at the final sample count for BO methods on four synthetic test functions.}
    \label{fig:dist-syn}
\end{figure}

\begin{figure*}[t]
    \centering
    \begin{minipage}[t]{0.48\columnwidth}
        \centering
        \includegraphics[width=\textwidth]{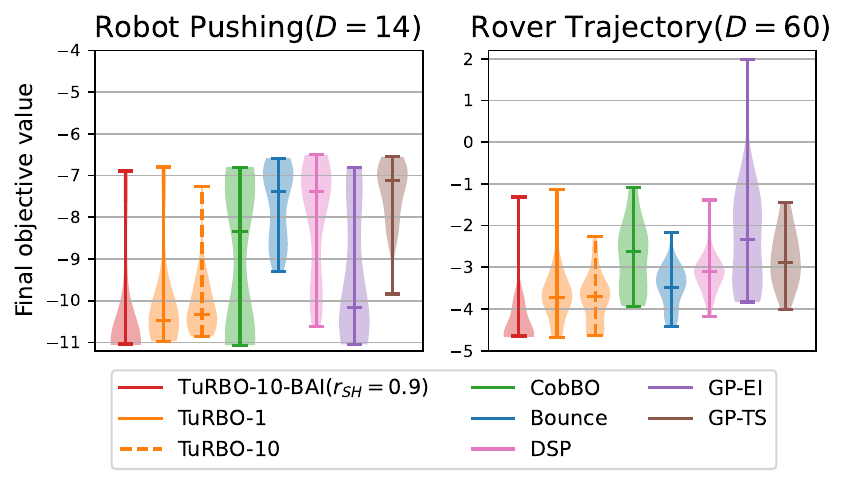}
    \end{minipage}
    \hfill
    \begin{minipage}[t]{0.48\columnwidth}
        \centering
        \includegraphics[width=\textwidth]{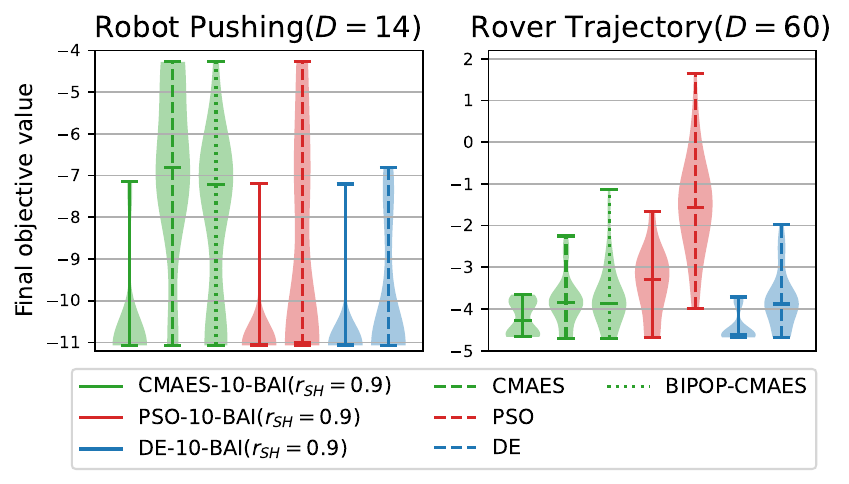}
    \end{minipage}
    \caption{Distribution of the best objective values at the final sample count on two real-world problems. Left: BO methods. Right: evolutionary algorithms.}
    \label{fig:dist-real}
\end{figure*}

\begin{figure}[t]
    \centering
    \includegraphics[width=0.95\textwidth]{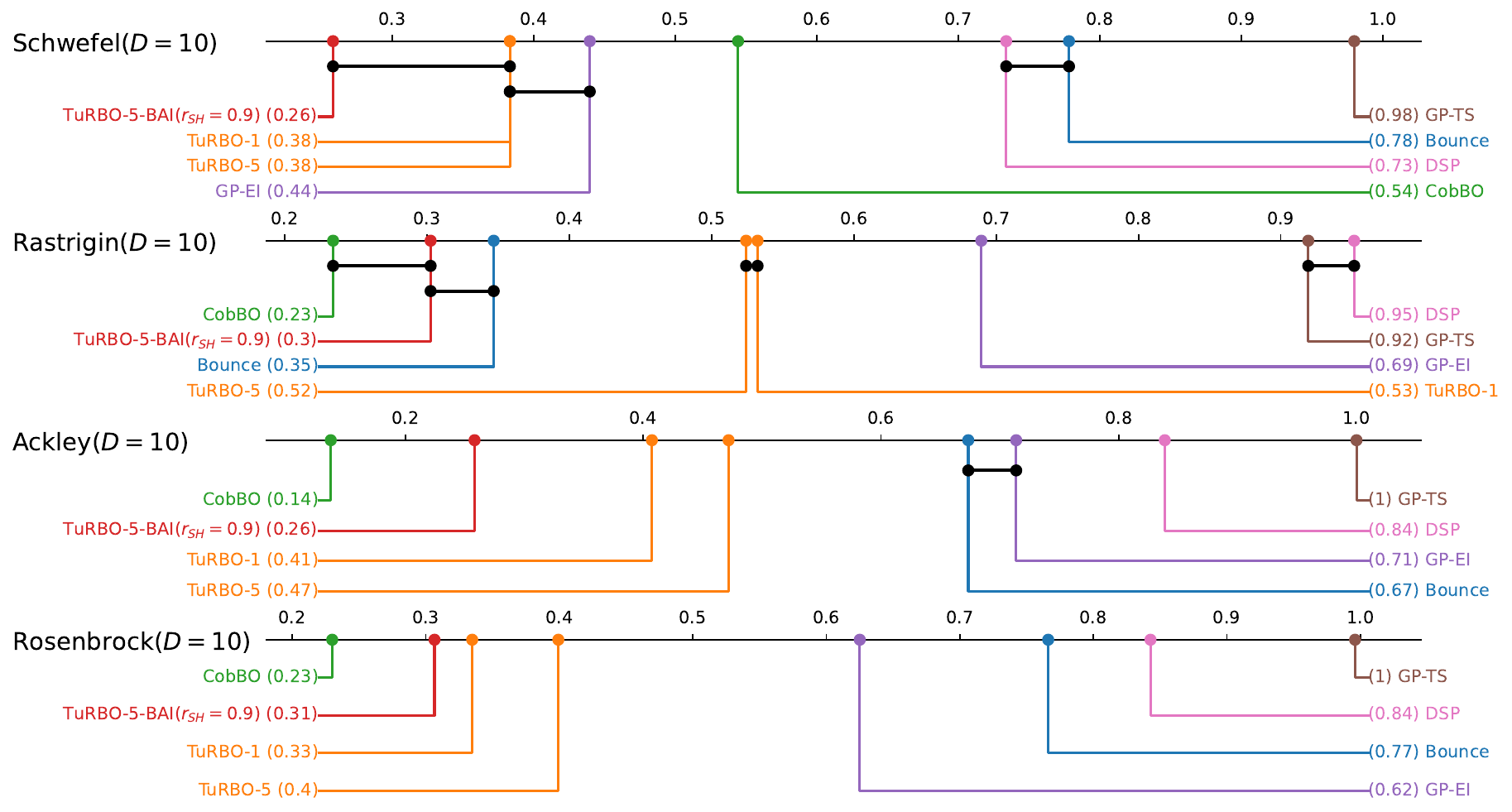}
    \caption{Critical difference diagram for BO methods on four synthetic test functions.}
    \label{fig:test-syn}
\end{figure}
\begin{figure}[t]
    \centering
    \includegraphics[width=0.95\textwidth]{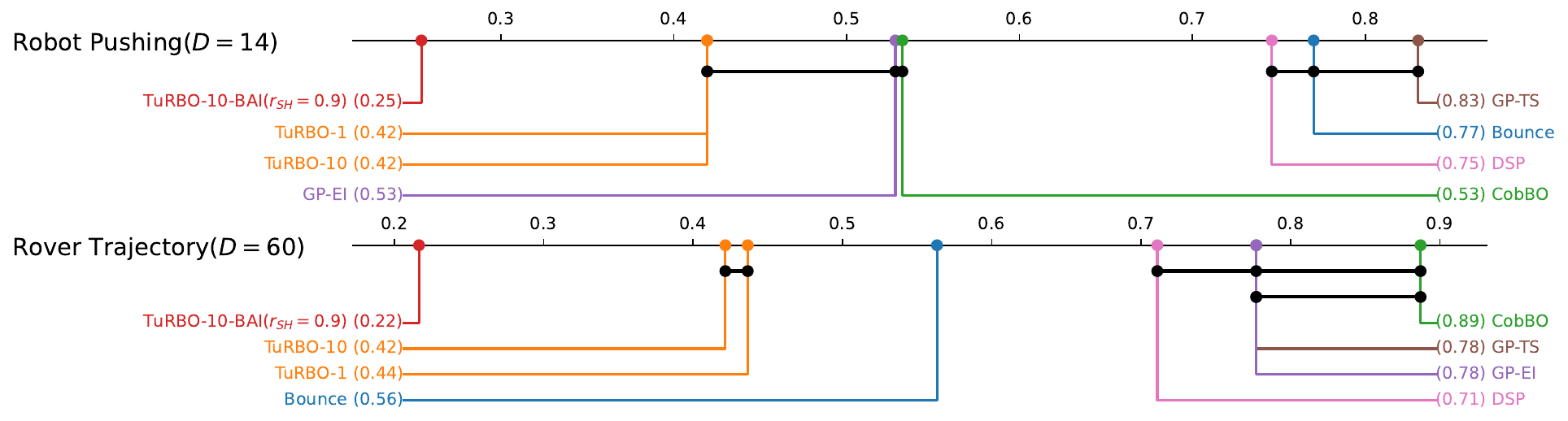}
    \caption{Critical difference diagram for BO methods on two real-world problems.}
    \label{fig:test-real}
\vspace{4mm}
    \centering
    \includegraphics[width=0.95\textwidth]{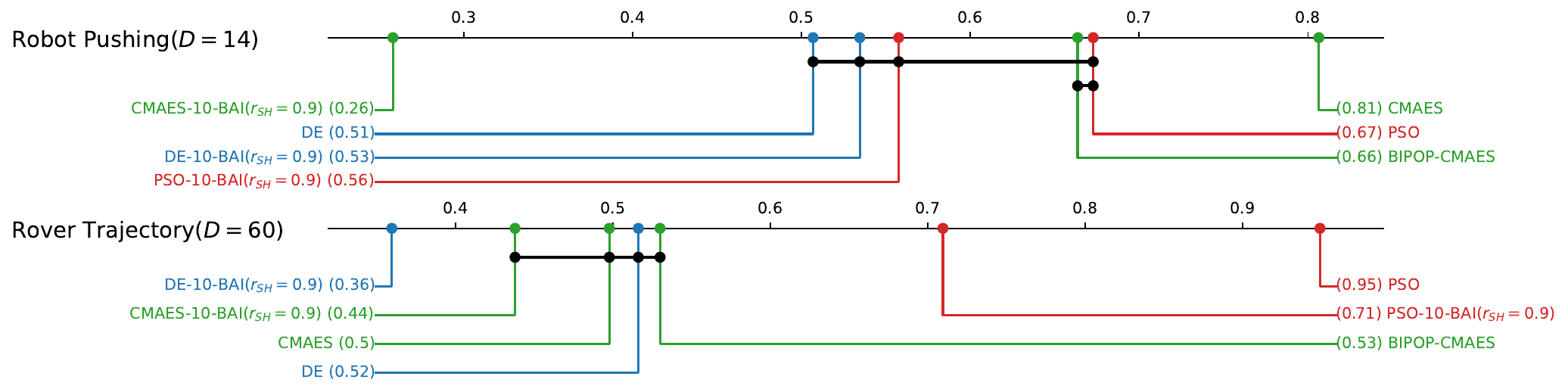}
    \caption{Critical difference diagram for evolutionary algorithms on two real-world problems.}
    \label{fig:test-real-ea}
\end{figure}

Figures \ref{fig:dist-syn} and \ref{fig:dist-real} show the distributions of the best objective values over 31 trials after all function evaluations have been completed.
Figures \ref{fig:test-syn}--\ref{fig:test-real-ea} show the results of the Wilcoxon rank-sum test at a significance level of $p=0.05$, comparing the best objective values achieved by each algorithm across trials. No correction for multiple comparisons was applied.

\paragraph{Synthetic Test Functions}
In Figs. \ref{fig:dist-syn}, TuRBO-5-BAI demonstrates superior performance, consistently outperforming both TuRBO-1 and TuRBO-5 across all tested problems. 
TuRBO-5-BAI shows statistically significant improvements over TuRBO-1 and TuRBO-5 on the Rastrigin, Ackley, and Rosenbrock functions.
On Schwefel, TuRBO-5-BAI significantly outperforms TuRBO-5, although no significant difference is observed relative to TuRBO-1. 
In contrast, standard TuRBO-5 underperforms TuRBO-1 in all cases, suggesting that simple TR parallelization does not contribute to global optimization.

Except for Rosenbrock, the synthetic test functions are separable, favoring methods like CobBO and Bounce that rely on dimension reduction. In particular, Bounce applies dimensionality reduction by assigning identical values to multiple variables, which leads to overfitting in settings where the optimal solution contains identical values across dimensions, as is the case here. This leads to misleadingly fast convergence on the Rastrigin function due to favorable initial samples, and thus should not be regarded as a fair comparison.

\paragraph{Real-World Problems}
In Fig. \ref{fig:dist-real}, TuRBO and the evolutionary algorithms equipped with the proposed trajectory-aware BAI converge near the global optimum in most trials, highlighting the effectiveness of BAI for global optimization.
As shown in Fig. \ref{fig:test-real}, TuRBO-10-BAI achieves significantly better results than all baselines.
CMAES-10-BAI shows significant improvements over both CMAES and BIPOP-CMAES on Robot Pushing, while PSO-10-BAI and DE-10-BAI show significant improvements on Rover Trajectory.

In the Robot Pushing problem, many baselines converge around $-7$, while in Rover Trajectory, convergence is often observed near $-3.5$, suggesting the existence of strong local optima. 
Additionally, CobBO appears more prone to getting stuck in local optima on real-world problems and functions like Schwefel with weaker global trends while it shows strong performance on functions with prominent global structures, such as Rastrigin, Ackley, and Rosenbrock.

\section{Computational Resources and Time}
\label{apx:time}

\begin{figure}[t]
    \centering
    \includegraphics[width=0.95\textwidth]{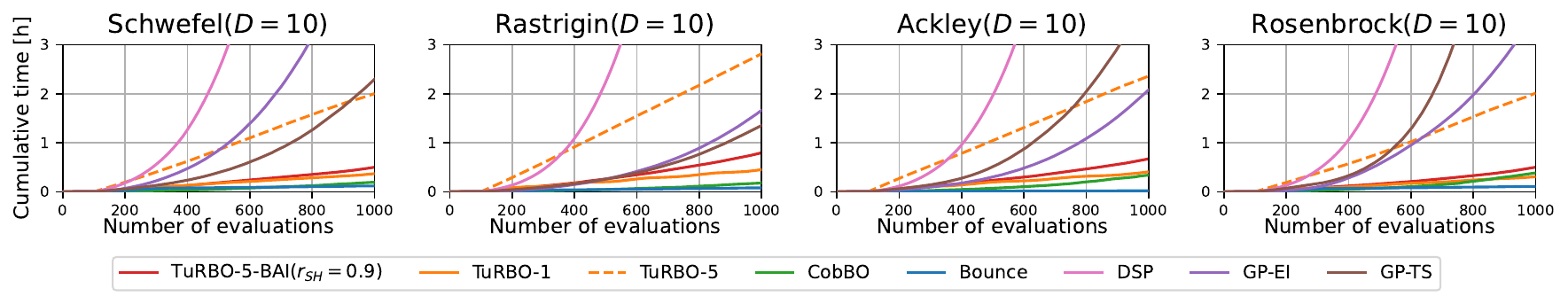}
    \caption{Computational time (hours) for BO methods on four synthetic test functions.}
    \label{fig:time-syn}
\end{figure}
\begin{figure*}[t]
    \centering
    \begin{minipage}[t]{0.48\columnwidth}
        \centering
        \includegraphics[width=\textwidth]{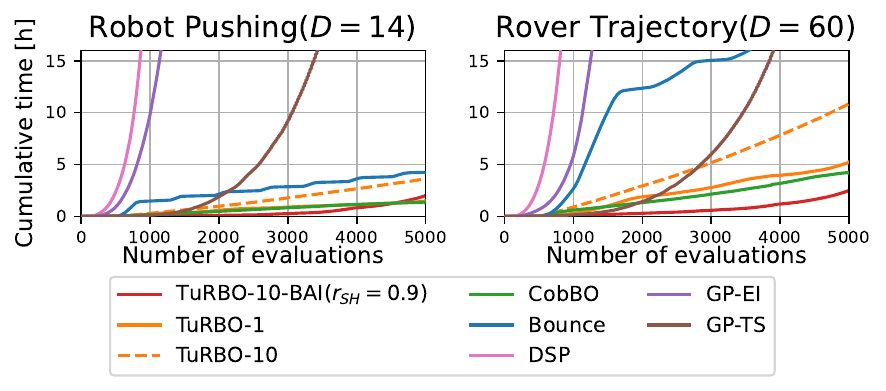}
    \end{minipage}
    \hfill
    \begin{minipage}[t]{0.48\columnwidth}
        \centering
        \includegraphics[width=\textwidth]{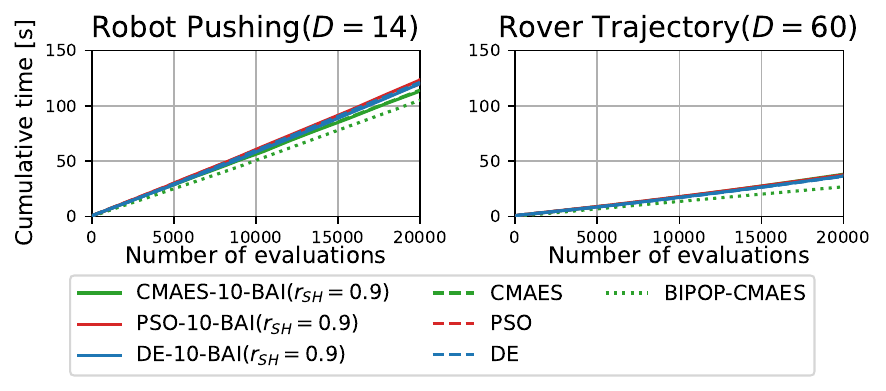}
    \end{minipage}
    \caption{Computational time on two real-world problems. Left: BO methods (hours). Right: evolutionary algorithms (seconds).}
    \label{fig:time-real-all}
\end{figure*}

The numerical experiments were performed using Intel Xeon Gold 6240 processors equipped with 256 GB of RAM. Figs \ref{fig:time-syn} and \ref{fig:time-real-all} present the average cumulative computation time over 31 trials plotted against the number of sample points evaluated for each algorithm.

TuRBO-$m$-BAI benefits from a reduction in the total number of samples added to each TR as $r_\text{SH}$ increases, which lowers the computational cost of GP regression that scales cubically with the number of samples, resulting in significantly reduced computation time. The computation time for TuRBO-1 depends on the frequency of restarts relatively short for synthetic test functions and the Robot Pushing problem, but somewhat longer for the Rover Trajectory problem. TuRBO-$m$ generally requires $m$ times the computation time of TuRBO-1 when the batch size $b$ is the same, leading to longer overall runtimes. CobBO exhibits consistently shorter computation times, while Bounce runs faster than CobBO on synthetic test functions but takes considerably longer on real-world problems. Methods employing global GPs, such as DSP, GP-EI, and GP-TS, show substantially longer computation times across all problems. In contrast, evolutionary algorithms, which do not use GPs, demonstrate overwhelmingly shorter computation times.

\section{Trajectory Prediction}
\label{apx:prediction}

\subsection{Ablation Study}
\label{apx:ablation}

The optimization trajectory prediction proposed in this work can be decomposed into the following five components:

\begin{enumerate}
    \item Use of the ridge regression model in Eq. \eqref{eq:ridge}
    \item Use of the best-so-far history $\mathcal{D}_{*}^{(j)}$
    \item Sorting of batched sample points (step 1 in Subsection \ref{subsec:trajectory})
    \item Elimination of initial sample points other than those with the minimum objective value (step 2 in Subsection \ref{subsec:trajectory})
    \item Elimination of sample points whose objective values are greater than or equal to the median of the entire best-so-far history (step 3 in Subsection \ref{subsec:trajectory})
\end{enumerate}

\begin{figure}[t]
    \centering
    \includegraphics[width=0.95\textwidth]{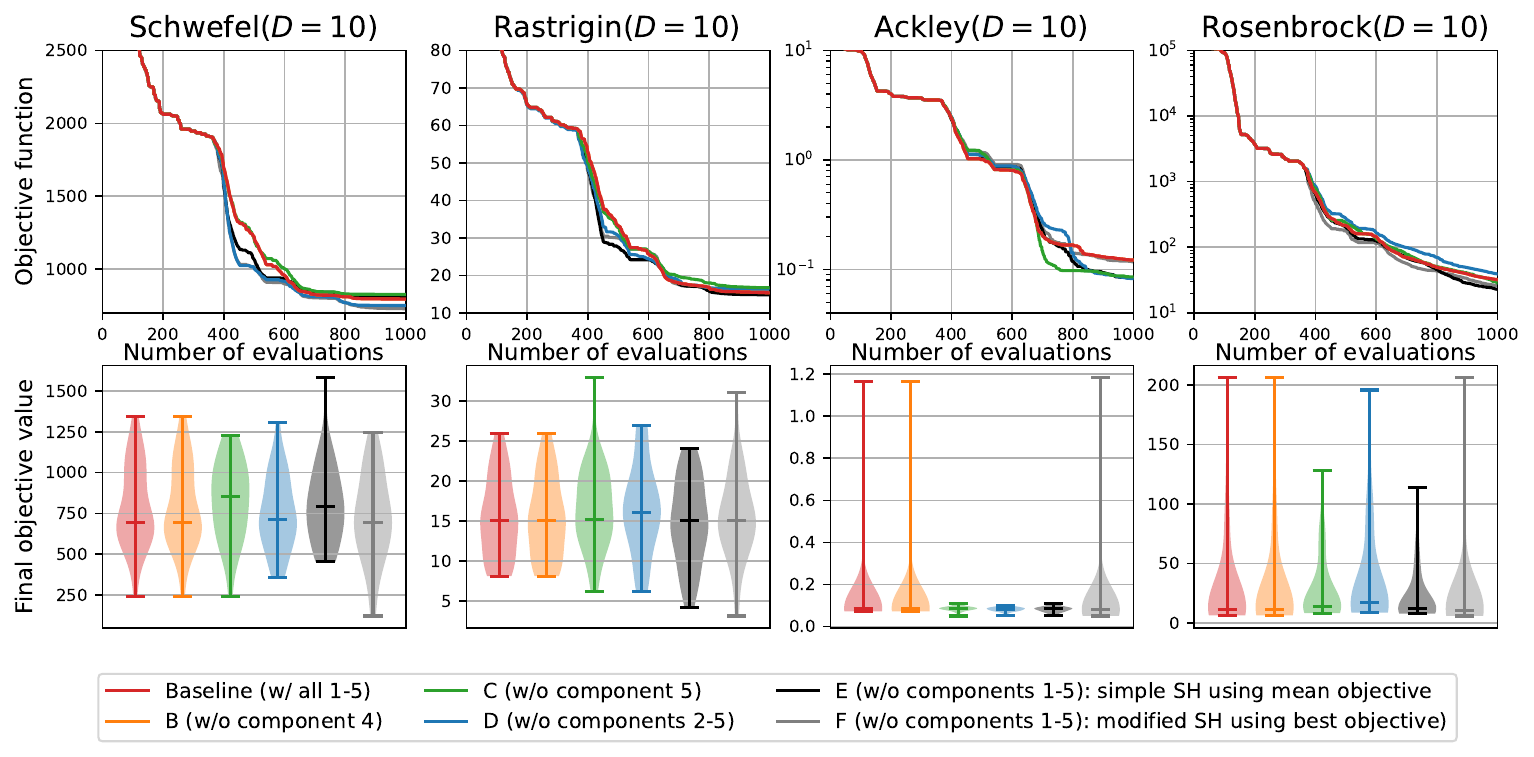}
    \caption{Ablation study of TuRBO-5-BAI with $r_\text{SH}=0.9$ on four synthetic test functions.}
    \label{fig:ablation-syn}
\end{figure}

\begin{figure}[t]
    \centering
    \includegraphics[width=0.48\columnwidth]{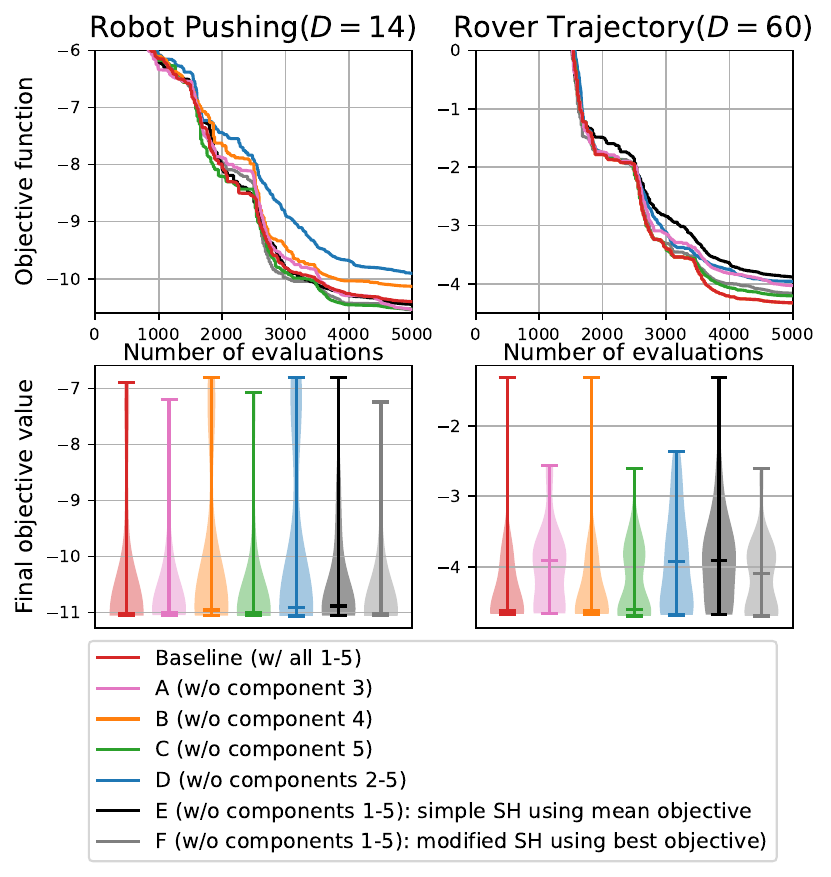}
    \caption{Ablation study of TuRBO-10-BAI with $r_\text{SH}=0.9$ on two real-world problems.}
    \label{fig:ablation-real}
\end{figure}

To evaluate the effectiveness of each component, we conduct an ablation study on four synthetic test functions and two real-world problems. Throughout the experiments, we use TuRBO-$m$-BAI with $r_{\text{SH}}=0.9$, which demonstrated the best overall performance in Appendix \ref{apx:proportion}, as the baseline.
The ablation study compares the baseline against six variants:
\begin{enumerate}[label*=\arabic*., labelwidth=4em, leftmargin=3em]
    \item[A--C.] One excluding each of components 3-5 individually
    \item[D.] One simultaneously excluding components 2-5, in which ridge regression is performed using all available samples $\mathcal{D}^{(j)}$ within each TR
    \item[E.] One excluding all components 1-5, where a simple SH is applied based on the average of all samples $\mathcal{D}^{(j)}$ within each TR
    \item[F.] One applying SH based solely on the minimum value in $\mathcal{D}^{(j)}$
\end{enumerate}
\noindent For the synthetic functions, since the batch size is $b=1$, the variant A that excludes component 3 is not applicable.

The experimental results are presented in Figs. \ref{fig:ablation-syn} and \ref{fig:ablation-real}. The impact of each component is most pronounced in the Rover Trajectory problem, where five variants (A, C, D, E, and F) show a performance drop compared to the baseline. Among these, variants A, D, and E exhibit statistically significant differences according to the Wilcoxon rank-sum test at the $p=0.05$ significance level. In the Robot Pushing problem, while no statistically significant differences are observed, variants B, D, and E show noticeable performance degradation.

For the synthetic functions, no significant differences are found between the baseline and any of the variants across all functions. Nonetheless, a slight drop in performance is observed for variants C and E in the Schwefel function and for variant D in the Rastrigin function. The relatively small influence of each component on the synthetic functions can be attributed to their globally dominant structure—except for the Schwefel function—making it possible to capture the overall trend even when using the average value of all sample points. This tendency is particularly evident in the Ackley function, which exhibits strong global characteristics; in such cases, variants C, D, and E, which use more sample points for prediction, show slight improvements over the baseline by better capturing the global trend.

Overall, the observed performance degradation from excluding components—especially in real-world problems—demonstrates the necessity of each component in our method.

\subsection{Prediction Example}
\label{apx:example}

\begin{figure}[t]
    \centering
    \includegraphics[width=0.9\textwidth]{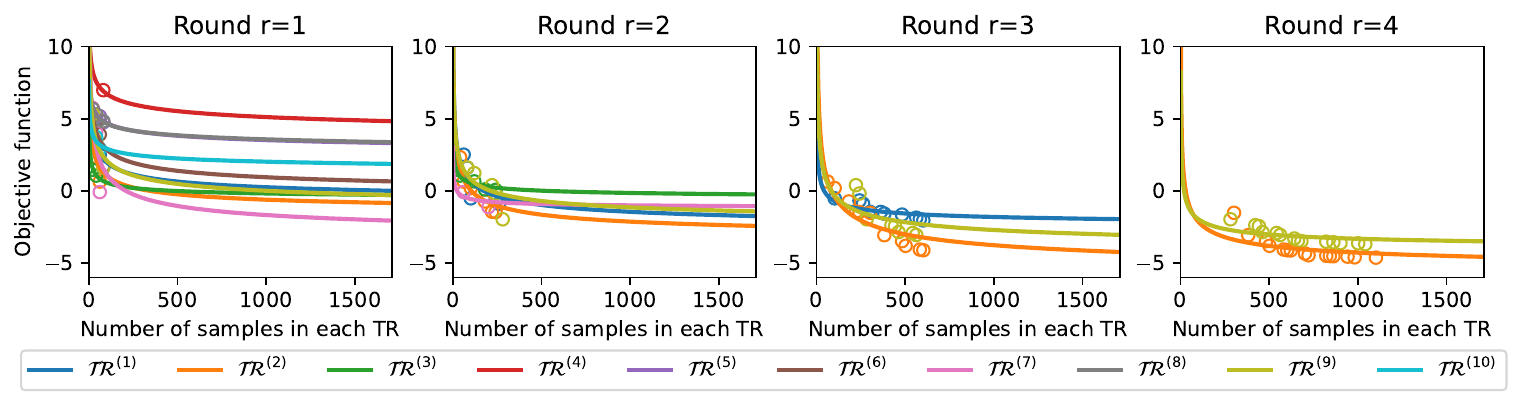}
    \caption{Examples of optimization trajectory prediction on the Rover Trajectory problem.}
    \label{fig:prediction}
\end{figure}

We verify how our trajectory-aware SH selects TRs by comparing the predicted objective functions obtained via Ridge regression in each TR. Fig. \ref{fig:prediction} shows the ridge regression models and the sample points used for training them at each round $r$ when applying TuRBO-10-BAI with $r_\text{SH} = 0.9$ to the Rover Trajectory problem. In this experiment, we use trajectory prediction based on the proposed method that incorporates all components 1–5 described earlier.

At round $r=1$, where 5 out of 10 TRs are selected, some TRs have only one sample point for regression. To prevent overfitting to such a small number of points, strong regularization with $\lambda=0.1$ proves effective. As the rounds progress, the number of TRs is halved, increasing the number of training samples available in each TR. From $r \geq 3$, the Ridge regression models exhibit slight underfitting, resulting in conservative predictions. 
While reducing $\lambda$ as $r$ increases could be a reasonable choice, setting appropriate values across diverse problems is nontrivial, and the benefit in the context of BAI appears limited \cite{takemori2021approximation}. Thus, we leave this as a direction for future work.

\subsection{Validation of Trajectory Prediction}
\label{sec:app-validation-of-trajectory-pred}
To evaluate the validity of the proposed trajectory prediction, we conducted an additional experiment using all optimization trials of TuRBO-$m$-BAI with $r_{\mathrm{SH}}=0.9$. At each BAI round $r$, we predicted the objective value that each surviving TR would achieve at the end of the following round $(r+1)$ and compared the prediction with the actual outcome. In the final two-arm stage, we predicted the final objective value obtained after the surviving TR consumed the remaining evaluation budget.

As a baseline, we also evaluated the trajectory prediction method of Takemori et al.~\cite{takemori24bai}, which approximately corresponds to the variant D in Appendix~\ref{apx:ablation} without Components 2, 4, and 5. Since pre-collected optimization histories were used in this experiment, Component 3 (sorting batched samples) was applied to both methods.

Table~\ref{tab:trajectory_prediction} summarizes the results. The proposed method consistently achieves substantially smaller RMSE than the prior method across all benchmark problems and BAI rounds. Moreover, the RMSE becomes sufficiently small in the later rounds, indicating that the predicted trajectories closely approximate the actual optimization outcomes. The proposed method also achieves higher Pearson and Spearman correlation coefficients in most settings, demonstrating that it more accurately preserves the ranking among candidate TRs.

The relatively low correlation observed in Round~1 is mainly caused by preprocessing Step~3 in Subsection~\ref{subsec:trajectory}. When no sample remains after threshold-based filtering, the filtering step is skipped to ensure that at least one sample is available for regression, resulting in a regression model with different characteristics from those used in later rounds. Nevertheless, the RMSE remains sufficiently small even in Round~1, supporting the validity of the proposed prediction. The Ackley function also exhibits relatively lower correlation coefficients because of its exponential objective function.

These results demonstrate the validity of our proposed approximation based on Assumption \ref{assumption:appendix}.2. In practice, as long as the regression preserves the same ranking as assumed in Assumption \ref{assumption:appendix}.2, the performance of TR selection is expected to be maintained.

\begin{table}[t]
\centering
\caption{Prediction accuracy of the proposed trajectory model and the baseline (BL) against the ground truth (GT) across 31 trials. Columns GT, Ours, and BL denote the ground-truth (GT) and predicted (Ours and BL) objective values at the next optimization round, averaged over all TRs that survive to the next round.}
\label{tab:trajectory_prediction}
\scriptsize
\resizebox{\textwidth}{!}{%
\begin{tabular}{llrrrrrrrrr}
\hline
Problem & Round & GT & Ours & BL & RMSE(Ours) & RMSE(BL) & Pearson(Ours) & Pearson(BL) & Spearman(Ours) & Spearman(BL) \\
\hline
Schwefel & 1 & 1286 & 1607 & 3199 & \textbf{737.9} & 1953 & -0.1724 & \textbf{0.6312} & -0.08703 & \textbf{0.5761} \\
Schwefel & 2 & 917.4 & 999.6 & 1063 & \textbf{183.6} & 367.4 & \textbf{0.8332} & 0.6033 & \textbf{0.8123} & 0.5432 \\
Schwefel & 3 & 796.4 & 726.9 & 298.3 & \textbf{81.52} & 510.0 & \textbf{0.9880} & 0.9413 & \textbf{0.9387} & 0.8456 \\
\hline
Rastrigin & 1 & 34.81 & 49.28 & 109.3 & \textbf{24.65} & 75.58 & -0.1550 & \textbf{0.5446} & -0.01771 & \textbf{0.4963} \\
Rastrigin & 2 & 18.82 & 28.71 & 37.79 & \textbf{11.92} & 22.19 & \textbf{0.6596} & 0.5554 & \textbf{0.6465} & 0.5383 \\
Rastrigin & 3 & 15.63 & 15.01 & 0.09153 & \textbf{1.233} & 16.27 & \textbf{0.9831} & 0.8852 & \textbf{0.9681} & 0.8554 \\
\hline
Ackley & 1 & 1.209 & 3.255 & 5.094 & \textbf{2.162} & 3.970 & 0.2794 & \textbf{0.3528} & \textbf{0.3238} & 0.3185 \\
Ackley & 2 & 0.1701 & 0.8340 & -0.5863 & \textbf{0.8255} & 0.9426 & \textbf{0.2293} & 0.01941 & \textbf{0.3441} & 0.2921 \\
Ackley & 3 & 0.08850 & -0.1447 & -1.754 & \textbf{0.3025} & 1.857 & \textbf{0.4433} & 0.2830 & \textbf{0.3547} & 0.1709 \\
\hline
Rosenbrock & 1 & 268.7 & -1299 & -2.131e5 & \textbf{6151} & 2.243e5 & -0.06108 & \textbf{0.06201} & 0.03943 & \textbf{0.1006} \\
Rosenbrock & 2 & 64.87 & -24.84 & -2.848e5 & \textbf{401.2} & 2.865e5 & -0.2158 & \textbf{0.2823} & 0.1660 & \textbf{0.1833} \\
Rosenbrock & 3 & 28.00 & 14.08 & -1.487e5 & \textbf{40.73} & 1.503e5 & \textbf{0.7710} & 0.2259 & \textbf{0.6091} & 0.08961 \\
\hline
Robot Pushing & 1 & -7.039 & -4.823 & -0.8045 & \textbf{2.624} & 6.398 & \textbf{0.3454} & 0.1263 & \textbf{0.3059} & 0.1573 \\
Robot Pushing & 2 & -8.595 & -6.742 & -3.971 & \textbf{2.175} & 4.996 & \textbf{0.7504} & 0.1977 & \textbf{0.7145} & 0.1483 \\
Robot Pushing & 3 & -9.786 & -8.599 & -8.363 & \textbf{1.320} & 1.924 & \textbf{0.9279} & 0.6413 & \textbf{0.7343} & 0.4936 \\
Robot Pushing & 4 & -10.61 & -9.896 & -11.36 & \textbf{0.7603} & 0.9772 & \textbf{0.9632} & 0.8773 & \textbf{0.7779} & 0.4974 \\
\hline
Rover Trajectory & 1 & -0.6301 & 1.230 & 3.185 & \textbf{2.199} & 4.046 & 0.3855 & \textbf{0.4099} & \textbf{0.4074} & 0.3464 \\
Rover Trajectory & 2 & -2.755 & -0.9874 & -0.3610 & \textbf{1.966} & 2.865 & \textbf{0.6113} & 0.1990 & \textbf{0.5988} & 0.1791 \\
Rover Trajectory & 3 & -3.786 & -2.933 & -3.826 & \textbf{1.032} & 1.439 & \textbf{0.7884} & 0.2798 & \textbf{0.7219} & 0.2833 \\
Rover Trajectory & 4 & -4.308 & -4.101 & -6.260 & \textbf{0.2701} & 2.031 & \textbf{0.9714} & 0.7670 & \textbf{0.5113} & 0.3678 \\
\hline
\end{tabular}%
}
\end{table}

\section{Parametric Study}
\label{apx:parametric}

\subsection{Proportion of Samples Used for Sequential Halving and Initialization: $r_{\text{SH}}$}
\label{apx:proportion}

\begin{figure*}[t]
    \centering
    \includegraphics[width=0.95\textwidth]{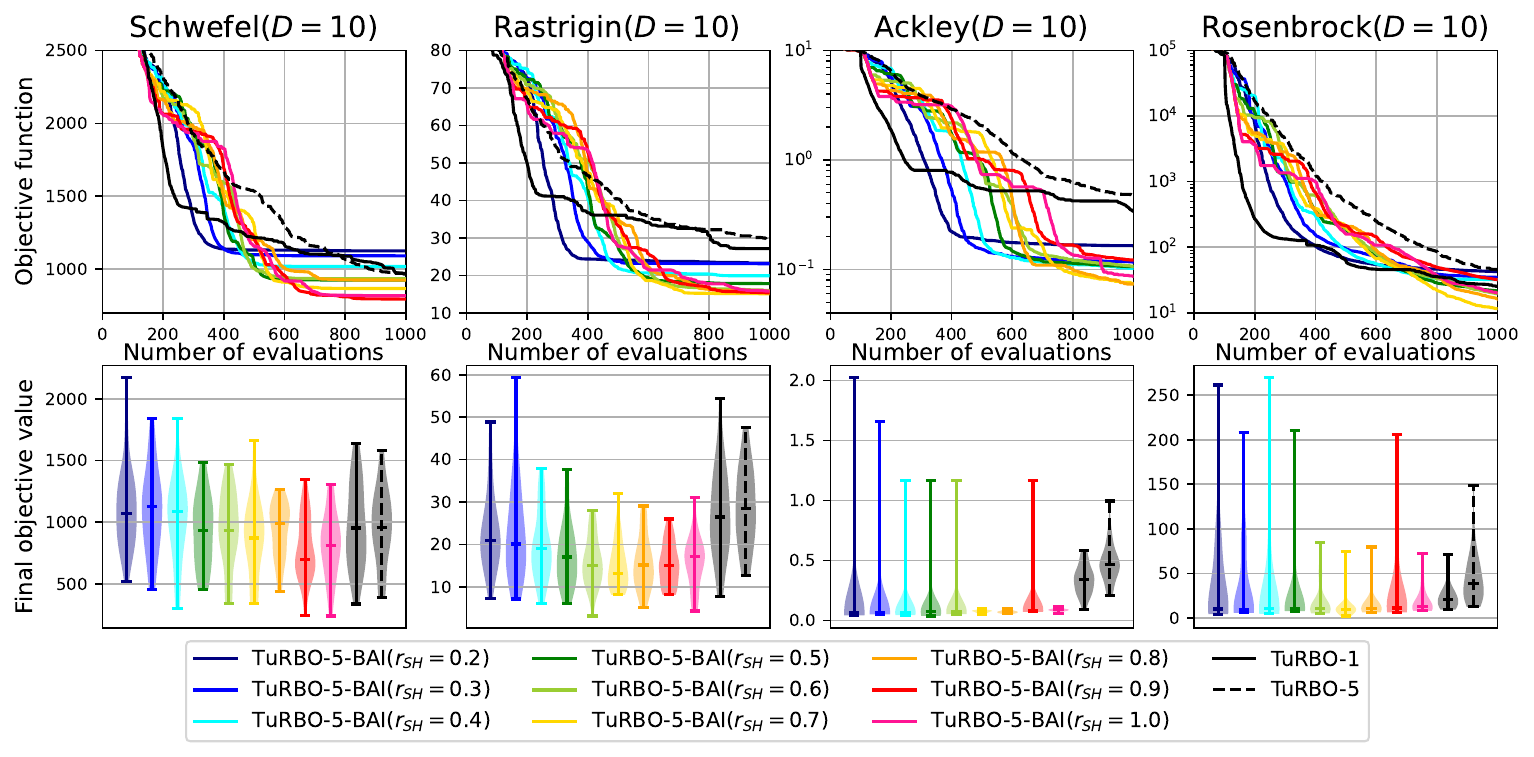}
    \caption{Effect of $r_{\text{SH}}$ for TuRBO-5-BAI on four synthetic test functions.}
    \label{fig:synthetic_rsh}
\end{figure*}

\begin{figure}[t]
    \centering
    \begin{minipage}[t]{0.49\textwidth}
        \centering
        \includegraphics[width=\linewidth]{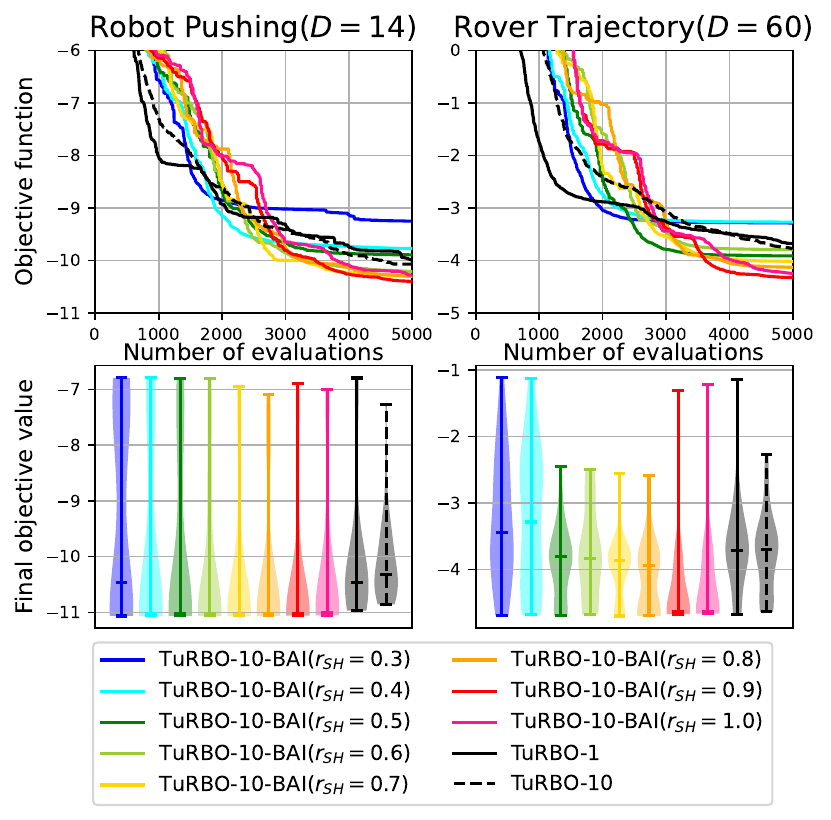}
        \captionof{figure}{Effect of $r_{\text{SH}}$ for TuRBO-10-BAI on two real-world problems.}
        \label{fig:real_rsh}
    \end{minipage}
    \hfill
    \begin{minipage}[t]{0.49\textwidth}
        \centering
        \includegraphics[width=\linewidth]{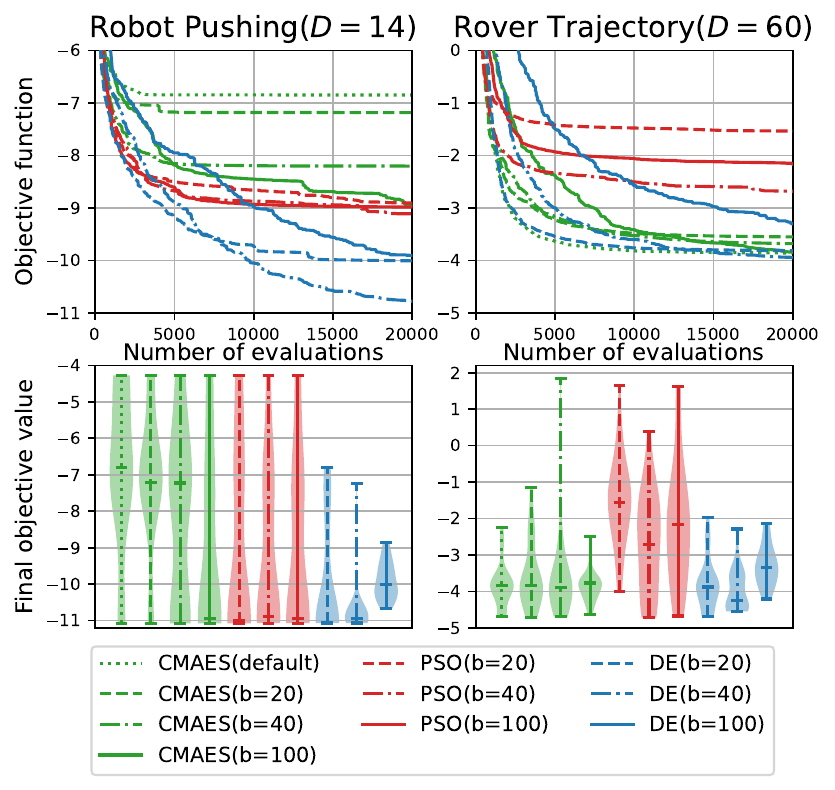}
        \captionof{figure}{Effect of population sizes of evolutionary algorithms on two real-world problems.}
        \label{fig:popsize}
    \end{minipage}
\end{figure}

The fraction of the total budget allocated to BAI is a key factor in improving optimizer performance with the proposed trajectory-aware BAI.
To investigate the effect of the budget allocated to SH in TuRBO-$m$-BAI, we vary $n_{\text{SH}}$ so that the total fraction of the budget used for initialization and SH satisfies
$r_{\text{SH}}=(n_{\text{SH}} + m \cdot n)/N \in \{0.2, 0.3, \dots, 1.0\}$.

The optimization results on the synthetic functions are shown in Fig. \ref{fig:synthetic_rsh}.
The best performance is achieved when $0.7 \leq r_\text{SH} \leq 1.0$, and within this range, TuRBO-$m$-BAI outperforms both TuRBO-1 and TuRBO-5 on all functions.
On Rastrigin, Ackley, and Rosenbrock, it outperforms TuRBO-1 and TuRBO-5 for all tested values of $r_\text{SH}$.
Figure \ref{fig:real_rsh} shows the results on the real-world problems.
Since $r_\text{SH} = 0.2$ does not provide enough samples for the SH procedure, we report results only for $r_\text{SH} \geq 0.3$.
In this setting, TuRBO-10-BAI with $0.9 \leq r_\text{SH} \leq 1.0$ consistently achieves the best performance on both problems.
Smaller values of $r_{\text{SH}}$ tend to lead to convergence to local optima, underscoring the importance of allocating a sufficient budget to arm identification.

These results empirically show that, when multiple optimizers are run in parallel under BAI, allocating 90--100\% of the total evaluation budget to BAI leads to more efficient convergence.

\subsection{Effect of Restart Suppression}

\begin{table}[t]
\centering
\caption{Effect of enabling trust-region restarts. Values in parentheses indicate the relative degradation from the reported results (restart disabled).}
\label{tab:restart}
\resizebox{\textwidth}{!}{
\begin{tabular}{lrrrrrr}
\hline
Method & Schwefel & Rastrigin & Ackley & Rosenbrock & Robot Pushing & Rover Trajectory \\
\hline
TuRBO-m-BAI ($r_{\rm SH}=0.3$) & 1093 (+0.1\%) & 23.16 (+0.2\%) & 0.1476 (+27.1\%) & 36.19 (+4.2\%) & -9.235 (+0.2\%) & -3.276 (+0.3\%) \\
TuRBO-m-BAI ($r_{\rm SH}=0.5$) & 926.2 (+0.2\%) & 17.89 (+0.4\%) & 0.1195 (+12.9\%) & 23.27 (+7.5\%) & -9.866 (+0.3\%) & -3.908 (+0.2\%) \\
TuRBO-m-BAI ($r_{\rm SH}=0.7$) & 875.4 (+1.0\%) & \textbf{15.61} (+3.3\%) & \textbf{0.09549} (+28.6\%) & \textbf{12.91} (+11.3\%) & -10.27 (+0.4\%) & -4.007 (+0.3\%) \\
TuRBO-m-BAI ($r_{\rm SH}=0.9$) & \textbf{797.1} (+0.3\%) & 15.74 (+2.5\%) & 0.1444 (+18.7\%) & 33.99 (+6.1\%) & \textbf{-10.36} (+0.4\%) & \textbf{-4.316} (+0.2\%) \\
TuRBO-m-BAI ($r_{\rm SH}=1.0$) & 837.9 (+2.4\%) & 16.88 (+6.9\%) & 0.1078 (+24.8\%) & 28.43 (+40.9\%) & -10.15 (+1.3\%) & -4.243 (+0.1\%) \\
\hline
TuRBO-1 & 970.4 & 27.15 & 0.3385 & 25.53 & -10.00 & -3.681 \\
TuRBO-m & 965.5 & 28.75 & 0.4754 & 45.74 & -10.08 & -3.775 \\
\hline
\end{tabular}
}
\end{table}

The proposed TuRBO-$m$-BAI disables trust-region restarts after initialization so that each TR consistently represents the same arm throughout the BAI procedure. A natural question is whether the observed performance improvements originate primarily from improved region identification or simply from suppressing restarts.

To investigate this effect, we performed an additional post-processing experiment. Whenever the standard TuRBO restart condition was satisfied, all subsequent improvements within that TR were ignored, simulating the behavior of restarting the TR. This produces a hypothetical ``restart = on'' result while keeping the optimization history before the restart unchanged. Table~\ref{tab:restart} reports the resulting average objective values together with the percentage degradation relative to the results reported in the paper (restart = off).

As shown in Table~\ref{tab:restart}, enabling restarts causes only marginal performance degradation on highly multimodal problems such as Schwefel, Rastrigin, Robot Pushing, and Rover Trajectory. Therefore, the superiority over TuRBO-1 remains essentially unchanged, indicating that the performance improvement on these problems mainly comes from more accurate region identification by BAI. In contrast, restart suppression becomes more important on weakly multimodal Ackley and unimodal Rosenbrock, where premature restarts interrupt local exploitation and noticeably degrade the final solution quality.

On the other hand, simply disabling restarts in standard TuRBO may lead to convergence to suboptimal local optima because the remaining evaluation budget cannot be reallocated to more promising regions. This behavior is observed for small values of $r_{\mathrm{SH}}$ (e.g., $r_{\mathrm{SH}} \leq 0.3$ in Figs.~\ref{fig:synthetic_rsh} and~\ref{fig:real_rsh}), where insufficient region identification results in inefficient budget allocation. Consequently, the proposed method combines two complementary mechanisms: BAI first identifies promising regions, and restart suppression then improves convergence within the selected region.

\subsection{Population Size for Evolutionary Algorithms}

Because trajectory-aware BAI runs multiple optimizer instances in parallel, relatively small population sizes are expected to be suitable when evolutionary algorithms are used as the base optimizers. On the other hand, a population size that is appropriate for BAI may be too small for standard evolutionary algorithms. In Section \ref{sec:experiment_ea}, we use the same population size for both standard evolutionary algorithms and their BAI-based variants. Here, to examine whether the performance gains of BAI can be reproduced simply by tuning the population size of standard evolutionary algorithms, we report the results of standard evolutionary algorithms with other population sizes.

Figure \ref{fig:popsize} shows that standard PSO and DE achieve their best performance with population size $b=40$, while CMAES performs well with $b=100$ on both problems. However, only DE with $b=40$ on the Robot Pushing problem reaches performance comparable to that of the BAI-based method in Fig. \ref{fig:real_ea}; in all other cases, the BAI-based variants obtain better solutions.
These results indicate that BAI can stably improve performance on multimodal problems for CMAES and PSO, which tend to emphasize local search. They also show that, even for DE, which is relatively strong at multimodal exploration, incorporating BAI achieves comparable or better performance.

\subsection{Number of Trust Regions and Batch Size: $m, b$}

\begin{figure}[t]
    \centering
    \includegraphics[width=0.95\textwidth]{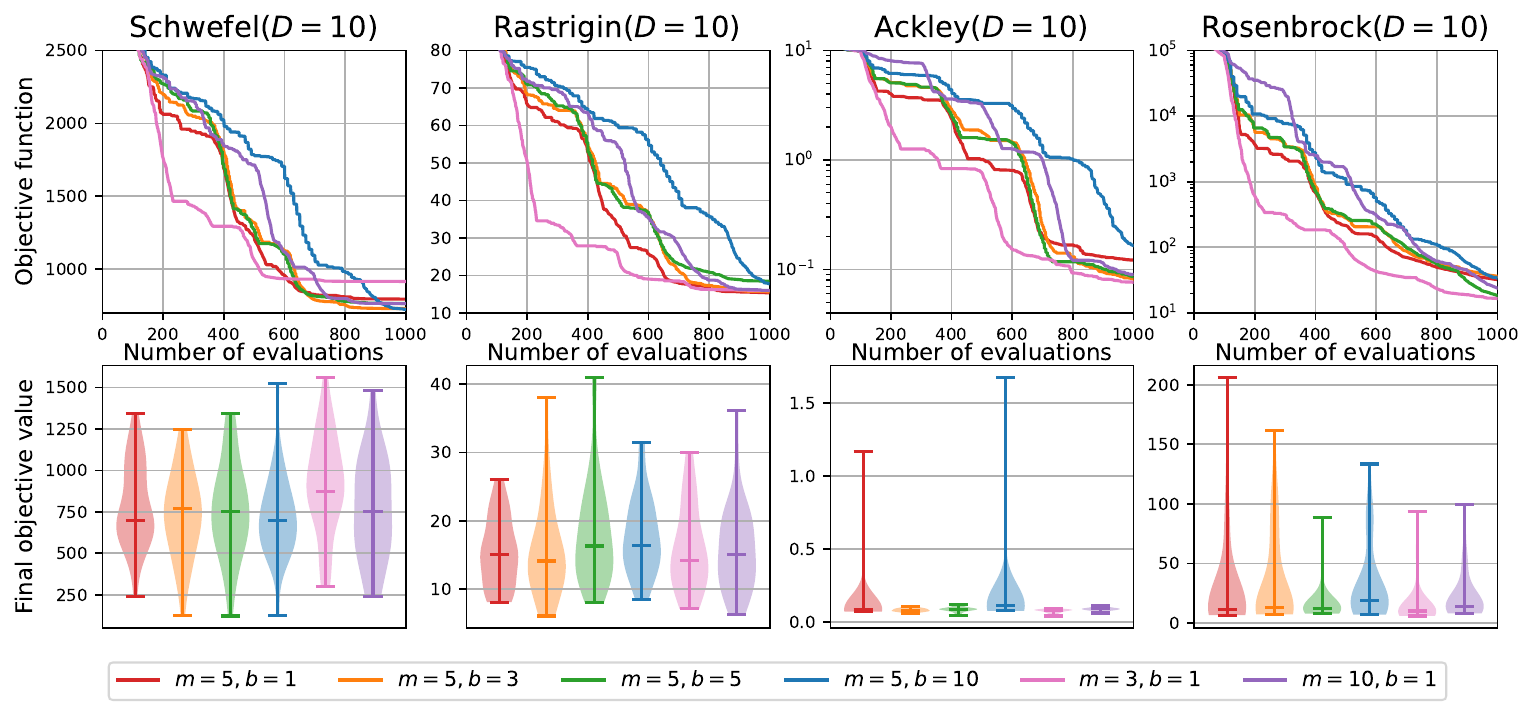}
    \caption{Parametric studies for the number of TRs $m$ and batch size $b$ around TuRBO-$m$-BAI with $r_\text{SH}=0.9$.}
    \label{fig:param-syn}
\end{figure}

In TuRBO-$m$-BAI, we hypothesize that there exists an appropriate number of TRs $m$ that depends on the complexity of the multimodal function. Additionally, while smaller batch sizes $b$ generally improve optimization performance, they increase the number of GP updates required, thereby incurring higher computational costs. To investigate the influence of these two parameters, we conduct experiments using four synthetic test functions.

Figure \ref{fig:param-syn} illustrates the performance of TuRBO-$m$-BAI with $r_{\text{SH}} = 0.9$, where $b=1$ and $m=5$ serve as the baseline setting. We vary the parameters across $b = \{1, 3, 5, 10\}$ and $m = \{3, 5, 10\}$ to assess their effects.
First, we fix $m = 5$ and analyze the impact of changing $b$. Across all test functions, increasing $b$ results in less frequent model updates, which slows down convergence. However, the final solution quality does not exhibit a consistent trend. Notably, in the Ackley and Rosenbrock functions, a relatively large performance drop is observed when using $b = 10$.

Next, we fix $b = 1$ and vary $m$. The effect of $m$ appears to be problem-dependent. In the Schwefel function, which features many scattered local optima near the boundaries of the domain, using $m = 3$ significantly degrades performance. Conversely, in simpler landscapes such as those of the Ackley and Rosenbrock functions, using $m = 3$ yields improved performance. The Rastrigin function, while featuring numerous local optima, also exhibits a pronounced global structure. Thus, its complexity can be considered intermediate between that of the Schwefel and Ackley functions. As a result, performance remains relatively stable across different values of $m$, with no substantial differences observed. These findings support the notion that more complex multimodal functions benefit from employing a larger number of TRs.

\section{Other benchmark Problems}
\label{apx:otherprob}

In Section \ref{sec:experiment_bo}, we evaluated our method on four synthetic test functions and two real-world problems. Here, we present additional results on a broader set of benchmark problems. Specifically, we consider the 10-dimensional Griewank ($[-40,60]^{10}$), Styblinski-Tang ($[-5,5]^{10}$), and Levy ($[-5,10]^{10}$) functions; the 6-dimensional Hartmann function ($[0,1]^6$); and the 5-dimensional Schwefel ($[-500,500]^5$), Rastrigin ($[-3,4]^5$), Ackley ($[-5,10]^5$), and Rosenbrock ($[-5,10]^5$) functions. For all experiments, we use $m = 5$ and $b = 1$. The total number of evaluations is set to $N=1000$ for 10-dimensional functionss and $N=500$ otherwise.

Figure \ref{fig:other} shows the results for the additional 6 and 10-dimensional synthetic functions. TuRBO-5-BAI demonstrates competitive performance across all problems, though the best-performing algorithm varies by function. For example, TuRBO-5-BAI excels on Griewank, DSP and GP-EI perform best on Styblinski–Tang, CobBO, Bounce, and GP-EI perform well on Levy, and Bounce and TuRBO-5 perform similarly well on Hartmann. While TuRBO-5-BAI is not always the top-performing method, it consistently achieves strong results when $0.6 \leq r_\text{SH} \leq 1.0$, underscoring its robustness.

Figure \ref{fig:synthetic5d} presents results on 5-dimensional versions of the same functions evaluated in Fig. \ref{fig:synthetic}. The trends remain consistent, indicating that comparable performance is achievable by appropriately scaling $N$ with dimensionality $D$. On the relatively simple Ackley landscape, TuRBO-5-BAI benefits from smaller $r_\text{SH}$ values, which favor continued local search without triggering restarts. However, $r_\text{SH} \leq 0.4$ may lead to entrapment in local optima, while performance degradation for larger $r_\text{SH}$ is minimal. Thus, we conclude that maintaining $0.9 \leq r_\text{SH} \leq 1.0$ offers a reliable balance across diverse problems.

\begin{figure}[t]
    \centering
    \includegraphics[width=0.95\textwidth]{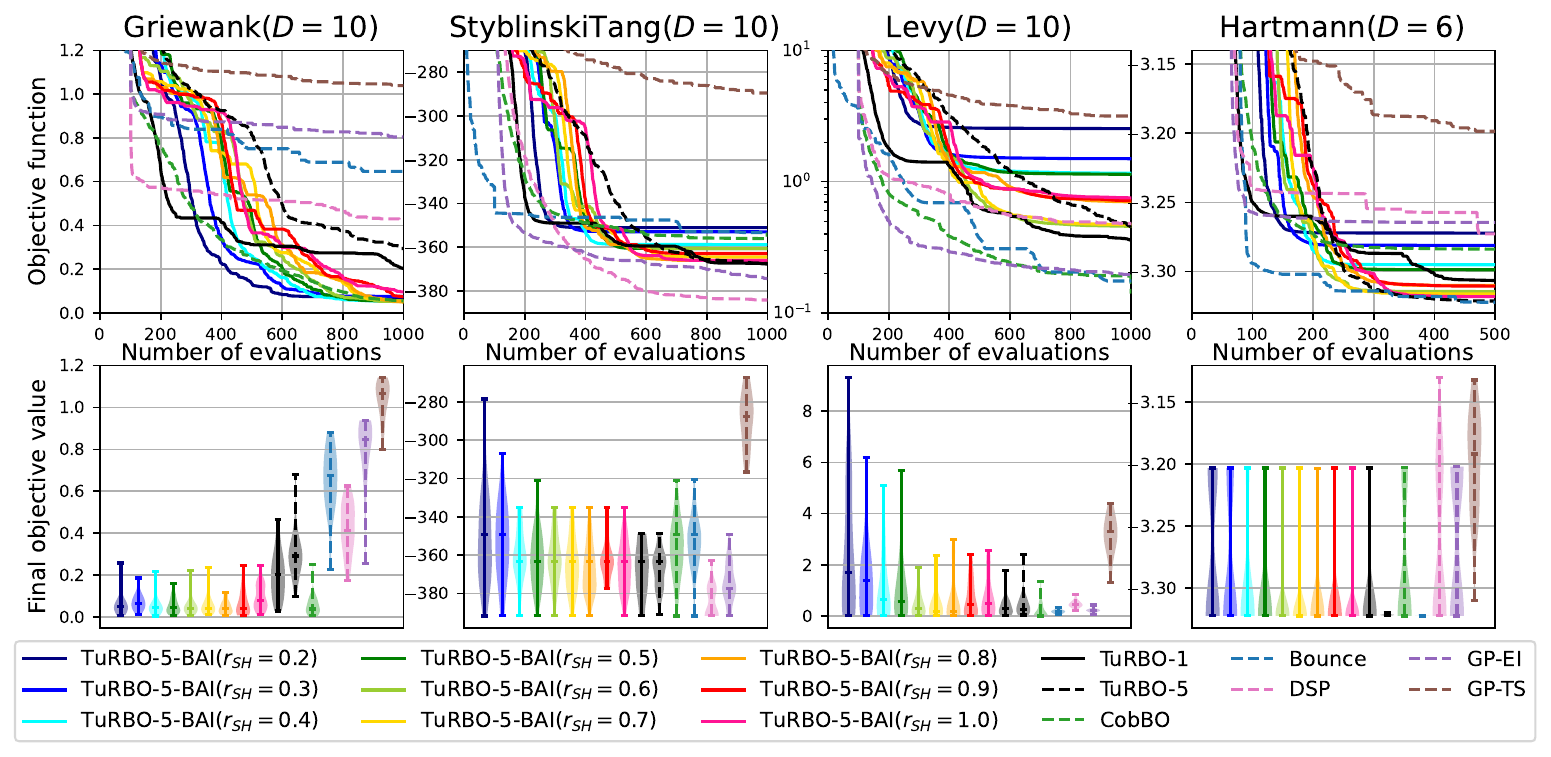}
    \caption{Comparison of TuRBO-5-BAI against baselines on additional synthetic test functions.}
    \label{fig:other}
\end{figure}

\begin{figure}[t]
    \centering
    \includegraphics[width=0.95\textwidth]{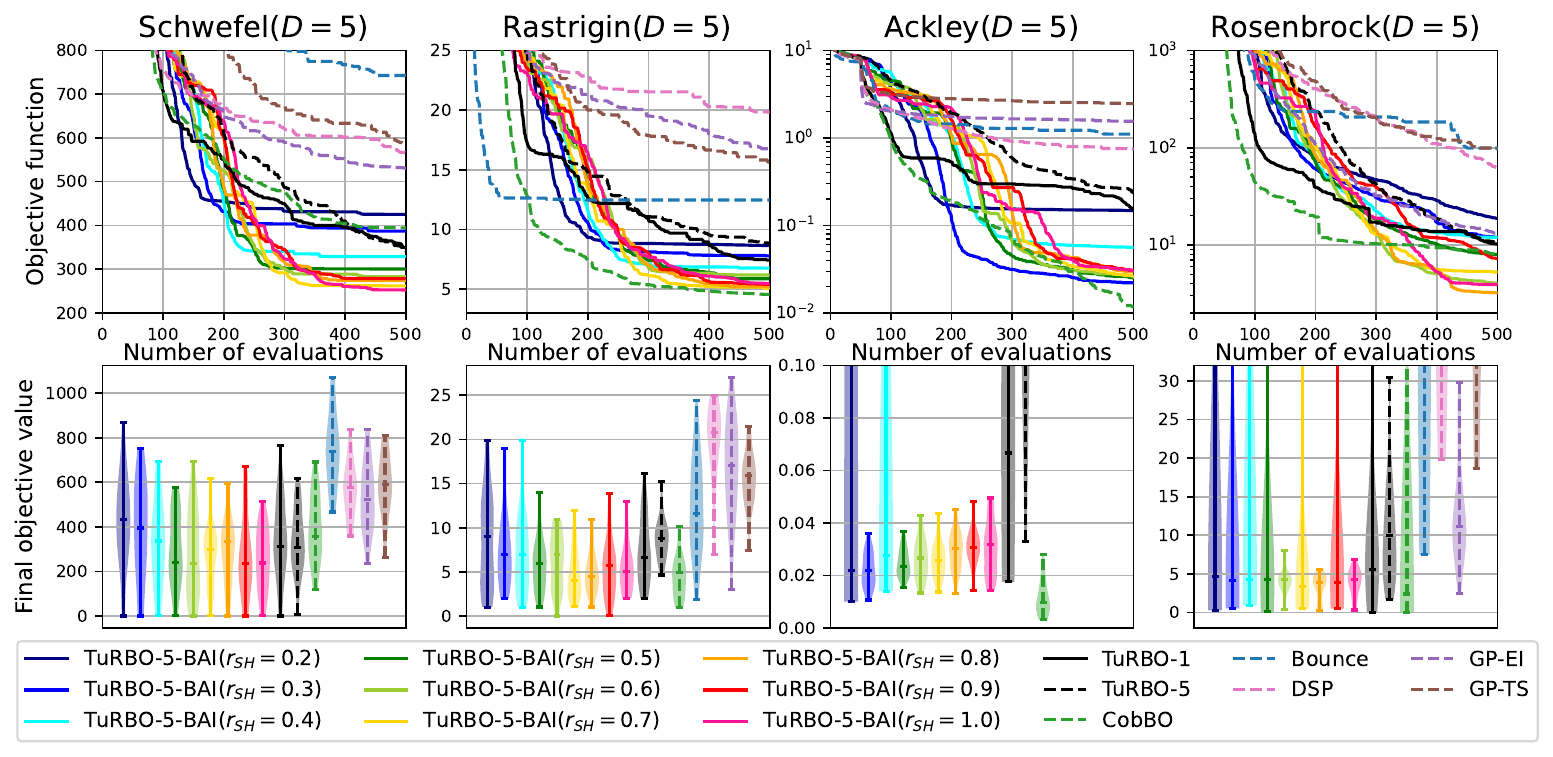}
    \caption{Comparison of TuRBO-5-BAI against baselines on 5-dimensional synthetic test functions.}
    \label{fig:synthetic5d}
\end{figure}



\end{document}